\documentclass[12pt]{article}

\usepackage{standalone}
\usepackage{arxiv}

\usepackage[utf8]{inputenc}%
\usepackage[T1]{fontenc}%
\usepackage[hidelinks]{hyperref}
\usepackage{booktabs}%
\usepackage{amsfonts}%
\usepackage{nicefrac}%
\usepackage{microtype}%
\usepackage{xcolor}%
\usepackage{dsfont}%
\usepackage{diagbox}
\usepackage{bm}
\usepackage{bbm}
\usepackage{eqparbox}
\usepackage[utf8]{inputenc}
\usepackage[english]{babel}
\usepackage{appendix}

\usepackage{graphicx}
\usepackage{subfig}

\usepackage{natbib}
\usepackage{amsmath}
\usepackage{amsthm}
\usepackage{amssymb}
\usepackage{comment}

\usepackage[linesnumbered,ruled,vlined,boxed]{algorithm2e}

\newcommand{\E}{\mathbb{E}}
\renewcommand{\P}{\mathbb{P}}
\newcommand{\R}{\mathbb{R}}

\DeclareMathOperator{\LS}{LS}
\DeclareMathOperator{\Tr}{Tr}

\newlength{\levlength}

\renewcommand{\epsilon}{\varepsilon}
\newcommand{\cG}{\mathcal G}

\newcommand{\cT}{\mathcal T}
\newcommand{\cA}{\mathcal A}

\newcommand{\cX}{\mathcal X}
\newcommand{\RR}{\mathbb{R}}
\newcommand{\cN}{\mathcal N}
\newcommand{\cF}{\mathcal F}
\DeclareMathOperator{\Span}{Span}

\newcommand{\alg}{\mathrm{Alg}}
\newcommand{\hati}{\hat \iota}

\renewcommand{\le}{\leqslant}
\renewcommand{\ge}{\geqslant}
\renewcommand{\leq}{\leqslant}
\renewcommand{\geq}{\geqslant}

\newcommand{\Texpl}{\cT_{\mathrm{expl}}}
\newcommand{\expl}{\mathrm{expl}}
\newcommand{\commit}{\mathrm{commit}}
\newcommand{\warmup}{\mathrm{warm-up}}
\newcommand{\algoname}[1]{\texttt{#1}}
\newcommand{\nous}{\algoname{E2TC}}
\SetKwInput{KwInit}{Init}

\theoremstyle{plain} \newtheorem{theorem}{Theorem}[section]
\newtheorem{remark}{Remark}[section] \newtheorem{definition} {Definition} [section]
\newtheorem{lemma} {Lemma} [section] \newtheorem{corollary} {Corollary} [section]
 \theoremstyle{definition}
\newtheorem{proposition}{Proposition}[section] 

\newtheorem*{lemma*}{Lemma}

\usepackage[nomain, stylemods=longextra, section=subsection, sort=def, style=long-name-desc-loc]{glossaries-extra}

\newglossary{environment}{env}{enl}{Notation related to problem setting}

\newglossary{algorithm}{algo}{agl}{Notation related to the algorithm and its analysis}

\newglossary{abbreviation}{abr}{abl}{Abbreviation used in the paper}

\makenoidxglossaries

\newglossaryentry{T}
{
    type=environment,
    name={\ensuremath{T}},
    description = {Time horizon},
}

\newglossaryentry{setT}
{
    type=environment,
    name={\ensuremath{[T]}},
    description = {Set of time step until $[T] = \{1,2,...,T\}$},
}

\newglossaryentry{dim}
{
    type=environment,
    name={\ensuremath{d}},
    description = {Dimension of the problem considered},
}

\newglossaryentry{A}
{
    type=environment,
    name={\ensuremath{A}},
    description = {Definite positive matrix defining a norm and an ellipsoid action set},
}
\newglossaryentry{lambdamax}
{
    type=environment,
    name={\ensuremath{\lambda_{\max}}},
    description = {Maximal eigenvalue of the matrix},
}

\newglossaryentry{lambdamin}
{
    type=environment,
    name={\ensuremath{\lambda_{\min}}},
    description = {Minimal eigenvalue of the matrix},
}
\newglossaryentry{norm}
{
    type=environment,
    name={\ensuremath{\|u\|_{M}}},
    description = {Norm defined by the matrix $M$, $\|u\|_{M} = \sqrt{u^\top M u}$ },
}

\newglossaryentry{actionset}
{
    type=environment,
    name={\ensuremath{\mathcal{X}}},
    description = {The action set $\mathcal{X} \subset \mathbb{R}^d$},
}

\newglossaryentry{center}
{
    type=environment,
    name={\ensuremath{c}},
    description = {Center of the ellipsoid that define the action set},
}

\newglossaryentry{theta}
{
    type=environment,
    name={\ensuremath{\theta}},
    description = {Unkown parameter of the model $\theta \in \mathbb{R}^d$},
}

\newglossaryentry{xstar}
{
    type=environment,
    name={\ensuremath{x^{\star}(\theta)}},
    description = {Optimal action as a function $\theta$ and $A$},
}

\newglossaryentry{actiont}
{
    type=environment,
    name={\ensuremath{x_t}},
    description = {Action taken at time $t$},
}

\newglossaryentry{rewardt}
{
    type=environment,
    name={\ensuremath{y_{t}}},
    description = {Reward received at time $t$},
}

\newglossaryentry{noiset}
{
    type=environment,
    name={\ensuremath{z_{t}}},
    description = {Noise at time $t$},
}

\newglossaryentry{noiselvl}
{
    type=environment,
    name={\ensuremath{\sigma}},
    description = {Subgaussian parameter of the noise},
}

\newglossaryentry{avregret}
{
    type=environment,
    name={\ensuremath{R_{T}(\theta)}},
    description = {Average regret at time $T$ on the environment parametrized by $\theta$},
}

\newglossaryentry{lower}
{
    type=algorithm,
    name={\ensuremath{\Omega}},
    description={Lowerbound Notation, "$f(\phi)= \Omega(g(\phi))$" iff $\exists c>0, \forall \phi, f(\phi) \geq c g(\phi)$},
}

\newglossaryentry{upper}
{
    type=algorithm,
    name={\ensuremath{O}},
    description={Upperbound Notation, "$f(\phi)= O(g(\phi))$" iff $\exists c>0, \forall \phi, f(\phi) \leq c g(\phi)$},
}

\newglossaryentry{forbidO}
{
    type=algorithm,
    name={\ensuremath{\widetilde{O}}},
    description={Upperbound Notation with hidden polylog terms},
}
\newglossaryentry{normtheta}
{
    type=algorithm,
    name={\ensuremath{B}},
    description={Norm $A$ of the parameter $\theta$},
}

\newglossaryentry{thetaset}
{
    type=algorithm,
    name={\ensuremath{\mathcal{E}_A(B)}},
    description = {Elliptical ensemble of vectors for which the norm is $B$ in $\|. \|_A$ norm},
}
\newglossaryentry{Algorithm}
{
    type=algorithm,
    name={\ensuremath{\mathcal{A}}},
    description={Generic notation for an algorithm},
}

\newglossaryentry{xrepa}
{
    type=algorithm,
    name={\ensuremath{w_t}},
    description={Reparametrization of the action $x_t$ centered and in base $A$, $x_t = A^{ \frac{1}{2}} w_t + c$},
}

\newglossaryentry{epsilondist}
{
    type=algorithm,
    name={\ensuremath{\epsilon^2}},
    description={Parameter in the proof to set the distance between two confusing problems' parameter},
}

\newglossaryentry{confusingfunc}
{
    type=algorithm,
    name={\ensuremath{\phi}},
    description={Function to create confusing parameters in the lowerbound},
}

\newglossaryentry{confusingvec}
{
    type=algorithm,
    name={\ensuremath{\xi}},
    description={Vector in $\xi \in \{-1,1\}^d$ to create confusing parameters in the lowerbound},
}

\newglossaryentry{confusetheta}
{
    type=algorithm,
    name={\ensuremath{\theta(\xi)}},
    description={Confusing parameters in the lowerbound},
}

\newglossaryentry{directstoptime}
{
    type=algorithm,
    name={\ensuremath{\tau_i}},
    description={Stopping time to control relative entropy in the $i$-th direction},
}

\newglossaryentry{directregret}
{
    type=algorithm,
    name={\ensuremath{U_i}},
    description={Random variable representing directional regret in the $i$-th direction},
}

\newglossaryentry{flipcoordinate}
{
    type=algorithm,
    name={\ensuremath{\xi^{(-i)}}},
    description={Flip the ith coordinate of the vector $\xi$. $ \forall j,
\xi^{(-i)}_j := (-1)^{ \mathbbm{1}(i=j) } \xi_j 
$ },
}
\newglossaryentry{probaithetadistib}
{
    type=algorithm,
    name={\ensuremath{P_{\theta(\xi),i}}},
    description={the probability distribution of $(w_t)_{t \le \tau_i(T)}$
    under parameter $\theta(\xi)$},
}

\newglossaryentry{constKL}
{
    type=algorithm,
    name={\ensuremath{C}},
    description={Parameter used to set a constant relative entropy between problems},
}

\newglossaryentry{kl}
{
    type=algorithm,
    name={\ensuremath{D}},
    description={Relative entropy between two distributions},
}

\newglossaryentry{Leastsquare}
{
    type=algorithm,
    name={\ensuremath{\LS}},
    description={Least-squares estimator with observed features $X$ and rewards $Y$. $\LS(X,Y) = (X^\top X)^{-1} X^\top Y $},
}

\newglossaryentry{design}
{
    type=algorithm,
    name={\ensuremath{X^{\top}X}},
    description={The design matrix with the oberving features $X \in \RR^{n \times d}$},
}

\newglossaryentry{setofinstant}
{
    type=algorithm,
    name={\ensuremath{\mathcal{T}}},
    description={Set of instant $\mathcal{T} \subset \mathbb{N}$},
}

\newglossaryentry{subsetfeature}
{
    type=algorithm,
    name={\ensuremath{X_{\mathcal{T}}}},
    description={Submatrix of the features of the reward $X_{\mathcal{T}} := (x_{t})_{t \in \mathcal{T}} \in \mathbb{R}^{|\mathcal{T}| \times d }$},
}

\newglossaryentry{subsetreward}
{
    type=algorithm,
    name={\ensuremath{Y_{\mathcal{T}}}},
    description={Subvector of the reward $Y_{\mathcal{T}} := (y_{t})_{t \in \mathcal{T}} \in \mathbb{R}^{|\mathcal{T}|}$},
}

\newglossaryentry{subsetnoise}
{
    type=algorithm,
    name={\ensuremath{Z_{\mathcal{T}}}},
    description={Subvector of the noise $Z_{\mathcal{T}} := (z_{t})_{t \in \mathcal{T}} \in \mathbb{R}^{|\mathcal{T}|}$},
}

\newglossaryentry{hattheta_i}
{
    type=algorithm,
    name={\ensuremath{\hat{\theta}_i}},
    description={The Least square estimation of the parameter $\theta$ at the end of phase $i$},
}

\newglossaryentry{alpha}
{
    type=algorithm,
    name={\ensuremath{\alpha}},
    description={Parameter used in the warm-up phase},
}

\newglossaryentry{hati}
{
    type=algorithm,
    name={\ensuremath{\hati}},
    description={Phase when the warm-up phase ends},
}

\newglossaryentry{hatB}
{
    type=algorithm,
    name={\ensuremath{\hat{B}}},
    description={The estimation of $\|\theta\|_A$ at the end of the warm-up phase, $\hat{B} = \|\hat{\theta}_{\hati}\|_A$},
}

\newglossaryentry{deltai}
{
    type=algorithm,
    name={\ensuremath{\delta_i}},
    description={Probability of failure of the $i$-th phase of the warm-up procedure},
}
\newglossaryentry{ni}
{
    type=algorithm,
    name={\ensuremath{n_i}},
    description={Duration of the $i$-th phase of the warm-up procedure},
}

\newglossaryentry{Ti}
{
    type=algorithm,
    name={\ensuremath{T_i}},
    description={Number of rounds after the $i$-th phase of the warm-up procedure},
}

\newglossaryentry{chiconf}
{
    type=algorithm,
    name={\ensuremath{U(\delta_i, n_i)}},
    description={Confidence bound of the estimated norm distribution with probability $1-\delta_i$ and $n_i$ samples. Related to the $\chi^2$ law.},
}

\newglossaryentry{Nexplore}
{
    type=algorithm,
    name={\ensuremath{N_e}},
    description={Number of exploration round during the exploration phase},
}

\newglossaryentry{hattheta}
{
    type=algorithm,
    name={\ensuremath{\hat{\theta}}},
    description={The least-squares estimation of the parameter $\theta$ at the beginning of the commit phase},
}

\newglossaryentry{Econdiexplor}
{
    type=algorithm,
    name={\ensuremath{\E_\expl}},
    description={Expectation conditionally on all rounds before phase 3},
}

\newglossaryentry{Ct}
{
    type=algorithm,
    name={\ensuremath{\mathcal{C}_t}},
    description={Confidence ellipsoid at time $t$ of the parameter $\theta$ in optimistic algorithms},
}

\newglossaryentry{logbar}
{
    type=algorithm,
    name={\ensuremath{\overline \log}},
    description={$\overline \log(x) = 1+\log(\max(1,x))$},
}

\newglossaryentry{moinsexp}
{
    type=algorithm,
    name={\ensuremath{-}},
    description={$a^{-} := \min(a,0)$},
}

\newglossaryentry{ETC}
{
    type=abbreviation,
    name={\algoname{ETC}},
    description={\algoname{Explore-Then-Commit}, Type of algorithm that first explores the environment then commits to the best actions},
}

\newglossaryentry{nous}
{
    type=abbreviation,
    name={\nous},
    description={\algoname{Explore-Explore-Then-Commit}, The locally asymptotic minimax algorithm introduced in this paper},
}

\newglossaryentry{CB}
{
    type=abbreviation,
    name={\algoname{CB}},
    description={\algoname{ConfidenceBall}, Optimistic based algorithm of \cite{dani_stochastic_2008}},
}

\newglossaryentry{PEGE}
{
    type=abbreviation,
    name={\algoname{PEGE}},
    description={\algoname{Phased Exploration and Greedy Exploitation} Nearly ETC like algorithm of \cite{rusmi_linearly_2010}},
}
\newglossaryentry{UE}
{
    type=abbreviation,
    name={\algoname{UE}},
    description={\algoname{Uncertainty
    Ellipsoid}, Optimistic algorithm of \cite{rusmi_linearly_2010}},
}

\newglossaryentry{OFUL}
{
    type=abbreviation,
    name={\algoname{OFUL}},
    description={\algoname{Optimism in the Face of Uncertainty Linear bandit}, Optimistic algorithm of \cite{abbasi2011improved}},
}

\newglossaryentry{OLSOFUL}
{
    type=abbreviation,
    name={\algoname{OLSOFUL}},
    description={\algoname{Ordinary Least Squares OFUL}, Optimistic algorithm of \cite{gales2022norm-agn}},
}

\newglossaryentry{TS}
{
    type=abbreviation,
    name={\algoname{TS}},
    description={\algoname{Thompson Sampling}, Sampling based algorithm of \cite{abeille_2017}},
}
\newcommand{\shorttitle}{Linear Bandits on Ellipsoids}
\title{Linear Bandits on Ellipsoids:\\ Minimax Optimal Algorithms}

\usepackage{authblk}

\author[1]{Raymond Zhang$\phantom{}^*$}
\author[1]{Hédi Hadiji$\phantom{}^*$}
\author[1]{Richard Combes}

\affil[1]{\centering Laboratoire des signaux et systèmes,
Univ. Paris-Saclay,
CNRS, 
CentraleSupélec 
\vspace{1em}
} 
\vspace{1em}

\date{\vspace{-1em}\today}

\begin{document}
\maketitle
\begingroup\renewcommand{\thefootnote}{}\footnotetext{$^*$ The first two authors contributed equally to this work.}\endgroup
\begin{abstract}
We consider linear stochastic bandits where the set of actions is an ellipsoid. 
We provide the first known minimax optimal algorithm for this problem. 
We first derive a novel information-theoretic lower bound on the regret of any algorithm, which must be at least $\Omega(\min(d \sigma \sqrt{T} + d \|\theta\|_{A}, \|\theta\|_{A} T))$ where $d$ is the dimension, $T$ the time horizon, $\sigma^2$ the noise variance, $A$ a matrix defining the set of actions and $\theta$ the vector of unknown parameters.
We then provide an algorithm whose regret matches this bound to a multiplicative universal constant. 
The algorithm is non-classical in the sense that it is not optimistic, and it is not a sampling algorithm. 
The main idea is to combine a novel sequential procedure to estimate $\|\theta\|$, followed by an explore-and-commit strategy informed by this estimate.
The algorithm is highly computationally efficient, and a run requires only time $O(dT + d^2 \log(T/d) + d^3)$ and memory $O(d^2)$, in contrast with known optimistic algorithms, which are not implementable in polynomial time.
We go beyond minimax optimality and show that our algorithm is locally asymptotically minimax optimal, a much stronger notion of optimality. 
We further provide numerical experiments to illustrate our theoretical findings. The code to reproduce the experiments is available at \url{https://github.com/RaymZhang/LinearBanditsEllipsoidsMinimaxCOLT}. 
\end{abstract}

\section{Introduction}\label{sec:Introduction}
\subsection{Model}\label{ssec:Model}
We consider the problem of stochastic linear bandits over ellipsoids with independent subgaussian rewards. 
Time is discrete with a horizon $\gls{T}$, and at time $t \in \gls{setT} = \{1,..,T\}$ a learner selects an action  $\gls{actiont} \in \mathcal{X}$ where action set $\mathcal{X}$ is a subset of $\mathbb{R}^{\gls{dim}}$ initially known to the learner. 
The action set $\mathcal{X}$ is an ellipsoid described by
\begin{align*}
	\gls{actionset} := \{ x \in \mathbb{R}^d: \|x-c\|_{A^{-1}} \le 1 \}
\end{align*}
with vector $\gls{center} \in \mathbb{R}^d$, positive definite matrix $\gls{A} \in \mathbb{R}^{d \times d}$ and $\gls{norm} := \sqrt{u^\top M u}$ for $u \in \mathbb{R}^d$ is the norm associated to positive definite matrix $M$.
After choosing action $x_t$ the learner observes a scalar reward $\gls{rewardt} = x_t^\top \theta + z_t$ where the vector $\gls{theta} \in \mathbb{R}^d$ is initially unknown to the learner, and $\gls{noiset} \in \mathbb{R}$ is additive noise.  
Random variables $(z_t)_{t \in \mathbb{N}}$ are assumed independent, centered and
subgaussian with variance proxy $\gls{noiselvl}^2$ so that conditionally on past observations
\begin{align*}
	\mathbb{E}\left(  \exp\left(\lambda z_t\right)\right) \le \exp\left(\frac{\lambda^2\sigma^2}{2}\right) \text{ for all } t \in \mathbb{N} \text{ and } \lambda \in \mathbb{R}
  \,.
\end{align*}
The learner selects action $x_t$ solely based on the observed rewards up to time $t$, that is, $y_1,...,y_{t-1}$.
The goal of the learner is to minimize the (expected) regret
\begin{align*}
	\gls{avregret} := \mathbb{E} \left[ \sum_{t=1}^{T} x^\star(\theta)^\top \theta - x_t^\top \theta \right]  \text{ where } \gls{xstar} \in \arg\max_{x \in \mathcal{X}} x^\top \theta
\end{align*}
which is the expectation of the difference between the rewards obtained by the learner and the reward obtained by an oracle who knows $\theta$ and hence always select the decision with the largest expected reward. 
From the Karush-Kush-Tucker conditions, $A^{-1} (x^\star(\theta)-c)$ should be proportional to $\theta$, so that the decision chosen by the oracle is $x^\star(\theta) = c + A\theta/\|\theta\|_{A}$.

For two functions $f,g$ of $\phi = (T,\theta,
\mathcal{X},d,\sigma)$, we write $f(\phi) = \gls{upper}(g(\phi))$ and $f(\phi) = \gls{lower}(g(\phi))$
if there exists $c_1,c_2$ strictly positive \emph{universal constants} (not allowed to depend on any parameter such as $\phi$) such that $f(\phi) \le c_1
g(\phi)$ and $f(\phi) \ge c_2 g(\phi)$ respectively, for all $\phi$.
We do not use the $\gls{forbidO}$ notation in which logarithmic terms are hidden, to avoid confusion. 

\subsection{Our contribution}\label{ssec:Our Contribution}

We make two main contributions to this problem. First we derive a novel
information-theoretic lower bound on the regret of any algorithm, stating that for any
algorithm and any $\gls{normtheta} \ge 0$
\begin{align*}
	\max_{\theta \in \mathcal{E}_{A}(B)} R_T(\theta) 
  = \Omega\left( \min(B T , d \sigma \sqrt{T})  \right)
  \text{ with } 
  \gls{thetaset}  := \{ \theta \in \mathbb{R}^d: \|\theta\|_{A} = B \}
  \,.
\end{align*}
In fact we prove an even stronger result showing that the same hold true when $\mathcal{E}_A(B)$ intersected by a neighbourhood of any $\theta$ of norm $ \|\theta\|_{A} = B$.  The size of the neighbourhood vanishes with $T$. 
Second, we propose a novel algorithm called~\gls{nous}~(\algoname{Explore-Explore-Then-Commit}) matching this upper bound so that under~\nous 
\begin{align*}
	\max_{\theta \in \mathcal{E}_A(B)} R_T(\theta) 
  = O\left( \min(B T , d \sigma \sqrt{T} + dB)  \right)
  \,.
\end{align*}
Algorithm~\nous~has low computational complexity (time $O(d T + d^2 \log(T/d) + d^3)$ and memory
$O(d^2)$) so it can be applied to high dimensions, and is not based on optimism or
sampling.
Rather, it combines a non-trivial, adaptive initialization procedure to estimate the
norm of $\theta$ followed by an Explore-Then-Commit (\gls{ETC}) strategy.
We emphasize that initially the learner has no information about the norm of $\theta$.
In fact not only do our results show that~\nous~is minimax optimal with low computational complexity, but they also show that~\nous~is locally asymptotically minimax optimal, a much stronger notion of optimality
than minimax optimality, further discussed next section.

\subsection{Locally asymptotic minimax optimality}
\label{ssec:Locally asymptotic minimax optimality}

\begin{figure}[!htbp]
    \centering    
    \subfloat[Minimax optimality]{\includegraphics[width=0.49\textwidth]{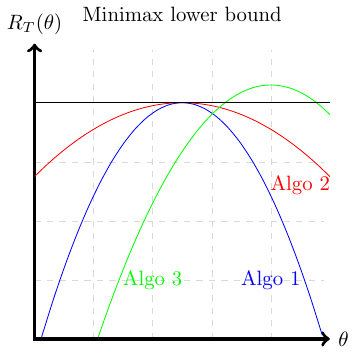}}
    \subfloat[Locally asymptotic minimax optimality]{\includegraphics[width=0.49\textwidth]{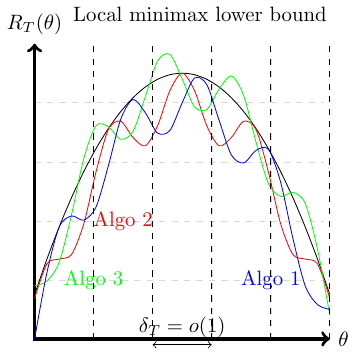}} 
    \caption{
      \scriptsize
      Schematic illustration of the differences between minimax  
      and locally asymptotically minimax lower bounds. \\
      Minimax lower bounds require that the regret curve of any algorithm goes above the
      bound at least once, whereas locally minimax lower bounds require that the curve of any
      algorithm goes above the bound \emph{in any region of size $\delta_T = o(1)$}. \\ In both figures, Algorithms 1 and 2 match the lower bound, 
       and Algorithm 3 does not.
    }
    \label{fig:minimax}
\end{figure}

We use a refined notion of optimality called "locally asymptotic minimax optimality" proposed by \cite{hajek1972local} for estimation. 
Denote by $R_T^\cA(\theta)$ the regret of algorithm $\gls{Algorithm}$ under parameter $\theta \in \mathbb{R}^d$ and let $\mathcal{P} \subset \mathbb{R}^d$ the set of allowed $\theta$.  
Algorithm $\cA$ is minimax optimal if for any algorithm $\cA'$ there exists $\theta \in \mathcal{P}$ such that with $R^{\cA'}_T(\theta) = \Omega( \max_{\theta \in \mathcal{P}} R^\cA_T(\theta))$.  
Algorithm $\cA$ is locally asymptotically minimax optimal if for any algorithm $\cA'$ there exists $\theta_T'$ such that $R^{\cA'}_T(\theta_T') = \Omega( R^{\cA}_T(\theta))$ and $\lim_{T \to \infty} \theta_T' = \theta$.  
Local asymptotic minimax optimality, as illustrated in Figure~\ref{fig:minimax}, is
stronger than minimax optimality. Locally optimal algorithms perform
well in the worst case, but also adapt to the information-theoretic "local
difficulty" of the problem.

For instance when $\mathcal{X}$ is the unit ball, and the set of allowed $\theta$ is $\mathcal{P} = \{\theta \in \mathbb{R}^d: \| \theta \|_{2} \le 1 \}$ \cite{lattimore_bandit_2020} prove the minimax lower bound  $\max_{\theta \in \mathcal{P}} R_{T}(\theta) = \Omega(d \sqrt{T})$, but their analysis only involves values of $\theta$ very close to $0$.
Our analysis paints a more complex picture, where the difficulty depends on $\|\theta\|_{A}$; locally minimax optimal algorithms like~\nous~adapt to this.

\subsection{Related Work}\label{ssec:Related Work}

We now summarize the current state of the art on the problem of stochastic linear
bandits with i.i.d.\ subgaussian rewards.
For fair comparison, we first highlight the difference in assumptions made by
authors, and whenever those assumptions differ from ours, we specialize their results
to our setting.

\textbf{Different settings} \cite{dani_stochastic_2008} consider an arbitrary action set $\mathcal{X}$ and assume that $\max_{x \in \mathcal{X}} |\theta^\top x| \le 1$, which in our setting is equivalent to $|c^\top \theta + \|\theta\|_{A}| \le 1$. 
\cite{rusmi_linearly_2010} consider an arbitrary compact action set $\mathcal{X}$, but they assume intricate conditions on its curvature, see below. 
\cite{abbasi2011improved} consider an arbitrary, possibly time-dependent, action set $\mathcal{X}_t$ and assume $\max_{x \in \mathcal{X}_t} |\theta^\top x | \le 1$, which in our setting is equivalent to $|c^\top \theta + \|\theta\|_{A}| \le 1$. 
They further assume that an upper bound for $\|\theta\|_{2}$ is known and provided as an input to the algorithm. 
\cite{abeille_2017} consider an action set $\mathcal{X}$ which is an arbitrary closed subset of the unit ball, and also assume that an upper bound for $\|\theta\|_{2}$ is known and provided as an input to the algorithm. 

\textbf{Regret upper bounds } The two versions of the \gls{CB} (\algoname{ConfidenceBall}) algorithm of~\cite{dani_stochastic_2008} achieve regret upper bounds of $O( d \sqrt{ T \log(T)^{3} })$ and $O( d^{3/2} \sqrt{ T \log(T)^{3} })$ in our setting.
The \gls{PEGE} (\algoname{Phased Exploration and Greedy Exploitation}) algorithm of~\cite{rusmi_linearly_2010} achieves a regret of $O( ( \|\theta\|_{2} + 1/\|\theta\|_{2} ) d \sqrt{T})$ when $\mathcal{X}$ is a unit ball. 
The \gls{UE} (\algoname{Uncertainty
Ellipsoid})  algorithm of \cite{rusmi_linearly_2010} achieves a regret of $O( f(\mathcal{X}) d \sqrt{T \log(T)^3})$ when $\mathcal{X}$ is an ellipsoid where $f$ depends on the curvature of $\mathcal{X}$ : $\lambda_{\min}(A) $,
The \gls{OFUL} (\algoname{Optimism in the Face of Uncertainty Linear bandit}) algorithm of~\cite{abbasi2011improved} achieves a regret $O( (\|\theta\|_{2} + \sigma \sqrt{d} \log(T) ) \sqrt{d T \log(T)})$ assuming that the algorithm is given an upper bound on $\|\theta\|_{2}$ as prior information in order to set the algorithm parameters correctly. 
The \gls{TS} (\algoname{Thompson Sampling}) algorithm of~\cite{abeille_2017} achieves a regret of $O( \sigma d^{3/2}\sqrt{T \log(dT)^3 })$ assuming that the algorithm is given an upper bound on $\|\theta\|_{2}$ as prior information in order to set the algorithm parameters correctly. 
The \gls{OLSOFUL} (\algoname{Ordinary Least Squares OFUL}) algorithm of~\cite{gales2022norm-agn} achieves a regret of $O( \sigma d \sqrt{ T \log(T/d)^2 } + \|\theta\|_{2}d \log(d) )$ without prior information on $\theta$. 
Therefore, the two versions of \algoname{CB}, \algoname{PEGE}, \algoname{UE}, \algoname{OFUL}, \algoname{OLSOFUL},  \algoname{TS} are not provably minimax optimal, in contrast with~\nous.

\textbf{Regret lower bounds } \cite{rusmi_linearly_2010} show that the minimax regret of any algorithm is at least $\Omega(d \sqrt{T})$ when the set of actions is the unit ball, using a Bayesian approach. 
\cite{lattimore_bandit_2020} show the same lower bound in a more general setting, using an Assouad-style construction. 
Our locally asymptotically minimax lower bound is a stronger result. with the correct
dependency on $d, \sigma, T$ and $\|\theta\|_{A}$. 

\textbf{Computational complexity} To the best of our knowledge, even when $\mathcal{X}$ is a unit ball, none of the optimistic algorithms such as \algoname{CB}, \algoname{UE}, \algoname{OFUL} and \algoname{OLSOFUL} cannot be implemented in polynomial time, because they must repeatedly maximize the bilinear function $x^\top \theta$ over the couple  $(x,\theta) \in \mathcal{X} \times \mathcal{C}_t$ where \gls{Ct} is the confidence ellipsoid computed at time $t$. 
To the best of our knowledge, no efficient algorithm is known to solve this bilinear optimization problem when $\mathcal{X}$ is an ellipsoid. 
The only case in which this is feasible in polynomial time is when $\mathcal{X}$ is finite with polynomial size in $d$. 
This is corroborated by our numerical experiments. 
Both \algoname{PEGE}, \algoname{TS} can be implemented in polynomial time, as they require to minimize a linear function over $\mathcal{X}$, and/or to compute the least-squares estimate of $\theta$. 
Our results show that there is no regret/complexity trade-off in linear bandits, in the sense that~\nous~achieves both minimax regret, and low computational complexity.  

\textbf{Related problems} 
Other authors have considered different but related problems in linear bandits, for
instance: \cite{zhu2022pareto} consider the case where $\theta$ is sparse,
\citet{banerjee2022exploration} lower bound asymptotically the eigenvalues of the
expected design matrix, \cite{jun2024noise} consider the case where the variance of the
reward from the optimal decision is small, and \cite{lattimore_bandit_2020} consider
the adversarial case (we refer the reader to their discussion of the major differences
between stochastic and adversarial linear bandits).

\section{Regret lower bound}\label{sec:Regret lower bound}

Our first main result is an information-theoretic lower bound on the regret of any
algorithm on ellipsoids, presented in Theorem~\ref{theorem:lower_bound}.
 Our result is a locally asymptotically minimax lower bound (stronger than a minimax
  lower bound) indeed, since given any $\theta$, there must
exist a $\theta'$ in the vicinity of $\theta$ on which regret is larger than the lower
bound, with $\|\theta - \theta'\| \to 0$ when $T$ is large. The lower bound is a
minimum of two terms, the first proportional to $d \sigma \sqrt{T}$, and the second to
$T \|\theta\|_{A}$, which depends on the norm of $\theta$. Those two terms identify two
"regimes" depending on $\|\theta\|_{A}$, suggesting that the correct measure of problem
complexity is $\|\theta\|_{A}$ which captures both the magnitude and the orientation of
$\theta$, as well as the curvature of the action set $\mathcal{X}$. 

Similarly to the minimax lower bound for the ball (Thm 24.2 in
\citet{lattimore_bandit_2020}), our proof is built around an Assouad-like construction
where we define $2^d$ alternative mean-reward vectors $\theta(\xi)$, with $\xi \in
\{-1,+1\}^d$, and argue that no algorithm can have low regret on all such parameters. 
We show our stronger result by enforcing that the alternative vectors be in the vicinity
of the given $\theta$, and of the same norm. This comes with technical difficulty, as we
need to lower bound the regret of playing actions that are close to optimal. We also took
extra care to exhibit clearly the different regimes involving $T, \sigma, B$ and $d$.
\begin{theorem}\label{theorem:lower_bound}
	For any algorithm, any $B > 0$, any $T \geq 1$ and any 
	$
	\theta \in \mathcal{E}_{A}(B) 
	= \left\{ \theta \in \mathbb{R}^{d+2}, \|\theta \|_{A} = B\right\}
	$ 
	there exists a bandit problem with mean-rewards vector 
	$\theta' \in \mathcal{E}_{A}(B)$ in a neighbourhood of $\theta$ 
	such that 
	$
	\|\theta - \theta'\|^2_A
	\le \min (\sigma dB /\sqrt{T} \, , 4 B^2)
	$ and :
	\begin{align*}
		R_T(\theta') \ge  
		\min \bigg( \frac{ \sigma d \sqrt{T}}{16} , \frac{B T}{4}\bigg) \,.
	\end{align*}
\end{theorem}

\noindent \textbf{Proof} \;
Let us set $B>0$ and $\theta \in \R^{d+2}$ such that $\|\theta\|_{A} = B$, and fix any 
algorithm.
We shall consider problems with gaussian rewards of variance $\sigma^2$ which are in
particular subgaussian with proxy variance $\sigma^2$.
For all $t$, define $\gls{xrepa} := A^{-1/2}(x_t - c)$, so that $x_t = A^{1/2} w_t + c$ 
and $w_t^\top w_t \le 1$.
The expected reward of action $x_t$ is then
$
\theta^\top x_t = ( A^{1 / 2} \theta)^\top w_t + \theta^\top c
$ 
and the reward of the optimal action is $B + \theta^\top c$ with $B = \sqrt{
\theta^\top A \theta }$. 
The difference between the two equals $B - ( A^{1/2} \theta )^\top w_t$.
 We work with 
 $(e_i)_{1 \le i \le d+2}$ an orthonormal basis of $\mathbb{R}^{d+2}$ 
 such that\footnote{The version of this paper published in the proceedings of COLT~2025 (PMLR vol.~291) and the previous arXiv version contain a typo where we used the inverse of $A$. With the present construction we indeed have a family of alternative parameters $\theta(\xi)$ in a neighborhood of $\theta$. The lower bound remains the same.}
 $e_{d+1}  \propto \textcolor{red}{A^{1/2}}\theta$
 and 
 $\Span(e_{d+1}, e_{d+2}) \supseteq \Span(A^{1/2}\theta, A^{-1/2}c)$.
We also set $\gls{constKL} := \min(\nicefrac{1}{2\sqrt{2}}, \nicefrac{B\sqrt{2T}}{\sigma d})$ and
$\gls{epsilondist} := \min(1, \nicefrac{d\sigma C}{B\sqrt{2T}})$. 
Given $\gls{confusingvec} \in \{-1,+1\}^d $ consider the family of problems defined by the mean-reward
vectors
\begin{align*}
	\gls{confusetheta} := A^{-\frac{1}{2}} \phi(\xi) \text{ and } 
	\gls{confusingfunc}(\xi) := B\left[ 
		 e_{d+1} \sqrt{1 - \epsilon^2} 
		 + \frac{\epsilon}{\sqrt{d}}\sum_{i=1}^{d} \xi_i e_i\right]
	\,.
\end{align*}
 Moreover, $\theta(\xi)^\top A \theta(\xi) = \phi(\xi)^\top \phi(\xi) = B^2$.   
By inspection, the regret at time step $t$ satisfies 
\begin{multline*}
	B - ( A^{\frac{1}{2} } \theta(\xi)  )^\top w_t  
	= B - \phi(\xi)^\top w_t 
	\stackrel{(1)}{=} \frac{B}{2} \bigg[ 
		1 - w_t^\top w_t + \Big\| \frac{\phi(\xi)}{B} - w_t\Big\|_2^2 
		\bigg] \\ 
	\stackrel{(2)}{\ge}  \frac{B}{2} \Big\| \frac{\phi(\xi)}{B} - w_t\Big\|_2^2 
	=  \frac{B}{2}  \bigg( 
		(w_{t,d+1} - \sqrt{1 - \epsilon^2} )^2 
		+  \sum_{i=1}^{d} \Big(w_{t,i} - \xi_i \frac{ \epsilon}{\sqrt{d}} \Big)^2 
		+ w_{t,d+2}^2 \bigg) 
  \,, 
\end{multline*} 
using the facts that (1) $ - a^\top b  = \frac{1}{2}(\|a-b\|^2 - a^\top a -b^\top b)$ and  (2) $w_t^\top w_t \le 1$. 
Summing over $t$ we get a lower bound for the regret:
\begin{align*}
	R_T(\theta(\xi)) 
	&= \mathbb{E}_{\theta(\xi)} \sum_{t=1}^{T}  \left[ B - \phi(\xi)^\top w_t \right] 
	\ge  \frac{B}{2} \mathbb{E}_{\theta(\xi)} \sum_{t=1}^{T} \left[  \left(w_{t,d+1} - \sqrt{1 - \epsilon^2} \right)^2 +  \sum_{i=1}^{d} \left(w_{t,i} - \xi_i \frac{ \epsilon}{\sqrt{d}} \right)^2 + w_{t,d+2}^2  \right]\\
	&\ge  \frac{B}{2}   \sum_{i=1}^{d} \mathbb{E}_{\theta(\xi)} \sum_{t=1}^{T} \left(w_{t,i} - \xi_i \frac{ \epsilon}{\sqrt{d}} \right)^2  
	\ge  \frac{B}{2}   \sum_{i=1}^{d} \mathbb{E}_{\theta(\xi)} ( U_i(\xi) )
	\hspace{3cm} \text{ where } \\
	\gls{directstoptime}(T)& := \min\bigg(  
		\min \bigg\{
			t: \sum_{s=1}^{t} w_{s,i}^2 \ge
			 \frac{d  \sigma^2 C^2}{2 \epsilon^2 B^2} 
			 \bigg\}, T 
		\bigg)
	 \text{ and } 
	 \gls{directregret}(\xi) := \sum_{t=1}^{\tau_i(T)} \left(w_{t,i} - \xi_i \frac{ \epsilon}{\sqrt{d}} \right)^2 
	 \,.
\end{align*} 
This is upper bounded with probability $1$ by:
\begin{align*}
	U_i(\xi) 
	&\le  2 \sum_{t=1}^{\tau_i(T)} \Big( w_{ti}^2 +  \frac{ \epsilon^2}{d} \Big) 
	\le  2 \left[\frac{d\sigma^2 C^2 }{2 \epsilon^2 B^2}  +   \frac{ T\epsilon^2}{d}\right]
  \,,
\end{align*}
using the Cauchy-Schwarz inequality, $(a-b)^2 \le 2(a^2 + b^2)$ and the definition of $\tau_i(T)$. For all $i \in [d]$ define $\gls{flipcoordinate}$ with coordinates 
$
\xi^{(-i)}_j := (-1)^{ \mathbbm{1}(i=j) } \xi_j 
$ 
and define $\gls{probaithetadistib}$ the probability distribution of $(w_t)_{t \le \tau_i(T)}$
under parameter $\theta(\xi)$. From the choice of $\theta(\xi)$ we have
\begin{align*}
	\theta(\xi)^\top x_t &-  \theta(\xi^{(-i)})^\top x_t
	= (\phi(\xi)- \phi(\xi^{(-i)}))^\top w_t + (\theta(\xi) - \theta(\xi^{(-i)}))^\top c \\
	&= (\phi(\xi)- \phi(\xi^{(-i)}))^\top w_t + \frac{2 \epsilon B}{\sqrt{d}} \xi_i e_i^\top A^{-\frac{1}{2}} c
	= (\phi(\xi)- \phi(\xi^{(-i)}))^\top w_t 
	= \frac{2 \epsilon B}{\sqrt{d}} \xi_i w_{t,i},
\end{align*}
as $e_i^\top A^{-1/2} c = 0$ for all $i \leq d$.
By \cite{lattimore_bandit_2020}[Exercises~14.7 and~15.8], the relative entropy
between the distributions of the observations in both problems bounded by
\begin{align*}
	\gls{kl}\big( P_{ \theta(\xi),i} &|  P_{ \theta(\xi^{(-i)}),i}  \big) 
	=\frac{1}{2\sigma ^2} \mathbb{E} \bigg( 
		\sum_{t=1}^{\tau_i(T)} 
			  \left( \theta(\xi)^\top x_t -  \theta(\xi^{(-i)})^\top x_t  \right)^2
		 \bigg) \\
	&= \frac{1}{2\sigma ^2} \mathbb{E} \bigg( 
		\sum_{t=1}^{\tau_i(T)} 
			\left( \phi(\xi)^\top w_t -  \phi(\xi^{(-i)})^\top w_t  \right)^2 
		\bigg) 
	=  \mathbb{E} \bigg( 
		\sum_{t=1}^{\tau_i(T)}w _{ti}^2 
		\bigg) 
		\frac{2 \epsilon^2 B^2 }{d\sigma^2}
	\le C^2 
	\,.
\end{align*}
Then because $U_i(\xi)$ is a function of $(w_t)_{t \le \tau_i(T)}$ and is bounded by
$ 2 \big(\frac{d\sigma^2 C^2 }{2 \epsilon^2  B^2}  +   \frac{ T\epsilon^2}{d}\big)$, we can use a total variation and Pinsker's inequalities to get that:
\begin{align*}
	\mathbb{E}_{\theta(\xi)}(	U_i(\xi) ) 
	&\ge   \mathbb{E}_{\theta(\xi^{(-i)})}(	U_i(\xi) ) - 2 \left[\frac{d\sigma^2 C^2 }{2 \epsilon^2 B^2}  +   \frac{ T\epsilon^2}{d}\right] \sqrt{\frac{1}{2} D( P_{ \theta(\xi),i} |  P_{ \theta(\xi^{(-i)}),i}  ) } \\
	&\ge   \mathbb{E}_{\theta(\xi^{(-i)})}(	U_i(\xi) ) -   \sqrt{2} \left[\frac{d \sigma^2 C^2 }{2 \epsilon^2 B^2}  +   \frac{ T\epsilon^2}{d}\right] C 
\,.
\end{align*}
Adding on both sides and using the definition of $U$:
\begin{align*}
	\mathbb{E}_{\theta(\xi)}( U_i(\xi) ) + \mathbb{E}_{\theta(\xi^{(-i)})}(	U_i(\xi^{(-i)}) ) 
	&\ge \mathbb{E}_{ \theta(\xi^{(-i)})} ( U_i(\xi)+ U_i(\xi^{(-i)}) ) -  
	\sqrt{2} \left[
		\frac{d\sigma^2 C^2 }{2 \epsilon^2 B^2}  
	+   \frac{ T\epsilon^2}{d}
	\right] C \\
	&=2 \mathbb{E}_{ \theta(\xi^{(-i)})}\bigg( 
		\sum_{t=1}^{\tau_i(T)} \frac{\epsilon^2}{d} + w_{ti}^2
		 \bigg) 
	- \sqrt{2} \left[\frac{d\sigma^2 C^2 }{2 \epsilon^2 B^2} 
	+   \frac{ T\epsilon^2}{d}
	\right] C \,.
\end{align*}
Either $\tau_i(T) = T$ so $\sum_{t=1}^{\tau_i(T)} \frac{\epsilon^2}{d} =
\frac{T\epsilon^2}{d}$, or $\tau_i(T) < T$ and then $ \sum_{t=1}^{\tau_i(T)} w_{ti}^2
\ge \frac{d\sigma^2C^2}{2 \epsilon^2 B^2}$. Therefore
\begin{align*}
	\mathbb{E}_{\theta(\xi)}(	U_i(\xi) ) + \mathbb{E}_{\theta(\xi^{(-i)})}(	U_i(\xi^{(-i)}) ) 
	&\ge \min \left(  2 T \frac{\epsilon^2}{d} , \frac{2d\sigma^2C^2}{ 2 \epsilon^2 B^2}\right) - \sqrt{2} \left[\frac{d\sigma^2 C^2 }{2 \epsilon^2 B^2}  +   \frac{ T\epsilon^2}{d}\right] C
  \,.
\end{align*}
If $d < 4B\sqrt{T} / \sigma $, then $C = \sqrt2 /4$ and 
$\epsilon^2 = d\sigma/( 4 B \sqrt{T}) < 1$ so
\begin{align*}
	\mathbb{E}_{\theta(\xi)}(	U_i(\xi) ) + \mathbb{E}_{\theta(\xi^{-i})}(	U_i(\xi^{(-i)}) )  
	&\ge \frac{ \sigma \sqrt{T}}{4 B} 
  \,.
\end{align*}
If $d \geq 4B \sqrt{T} /\sigma$, then $C = \sqrt{2T}B / (\sigma d)$
and $\epsilon^2 = 1$. So
$
2T / d
\leq d\sigma^2 C^2 /  B^2
$ and
\begin{align*}
	\mathbb{E}_{\theta(\xi)}(	U_i(\xi) ) + \mathbb{E}_{\theta(\xi^{(-i)})}(	U_i(\xi^{(-i)}) ) 
	&\ge     \frac{2T}{d} -  \sqrt{2} \frac{d \sigma^2 C^3 }{B^2} 
		=     \frac{2T}{d} -  \sqrt{2} \frac{(2T)^{3/2} B }{\sigma d^2} 
	\ge    \frac{T}{d}
  \,.
\end{align*}
Putting both cases together and summing the regret over $\xi$: 
\begin{align*}
	\sum_{\xi \in \{-1,+1\}^d}&  R_T(\theta(\xi)) 
	\ge  \frac{B}{2} \sum_{\xi \in \{-1,+1\}^d} \sum_{i=1}^{d} \mathbb{E}_{ \theta(\xi) }( U_i(\xi) )\\
	& \ge  \frac{B}{4} \sum_{i=1}^{d} \sum_{\xi \in \{-1,+1\}^d}  
	\mathbb{E}_{ \theta(\xi) }( U_i(\xi) ) + \mathbb{E}_{ \theta(\xi^{(-i)}) }( U_i(\xi^{(-i)}))
     \ge 2^d \min \bigg(  \frac{ \sigma d\sqrt{T}}{16} , \frac{TB}{4}\bigg)
 \,.
\end{align*}
We conclude the proof by saying that there must exist at least a $\xi$ such that
\begin{align*}
	R_T(\theta(\xi))
 &\ge \min\bigg(   \frac{ \sigma d\sqrt{T}}{16}  , \frac{TB}{4}\bigg) \qedhere
\end{align*}
To complete the picture, we also show that there is a minimal regret of $\Omega(d
\|\theta\|_A)$ as soon as $T \geq d$. While this seems intuitively straightforward, as
the learner should have no choice but to explore all dimensions at least once, the
proof (in Appendix~\ref{app:lower_bound}) requires some technical work. 
\begin{proposition}
\label{prop:dBlower}
 For any algorithm, for any $B > 0$,  if $T\geq d$, there exists a noiseless 
 bandit problem, with mean-rewards parameter $\theta \in \R^d$, 
 such that $\|\theta\|_A \leq 4 B$ and
	\[
		R_T(\theta) \geq 0.07 \, dB \geq 0.017 \, d \|\theta\|_A \,.
	\]	
\end{proposition}

\section{The~\nous~algorithm}
\label{sec:alg}

We now describe the~\nous~algorithm and prove its optimality. 
We first describe the algorithm and bound its regret when $\mathcal{X}$ is a centered ellipsoid i.e., $c=0$, and then extend our results to $c \ne 0$ using a reduction.  
To ease notation, we define $\gls{logbar} (x) := 1 + \log(\max (x, 1))$, $(a)^{\gls{moinsexp}} := \min (a,0)$, and given a set of instants $\gls{setofinstant} \subset \mathbb{N}$, we define vectors $\gls{subsetreward} := (y_{t})_{t \in \mathcal{T}} \in \mathbb{R}^{|\mathcal{T}|}$ and  $\gls{subsetnoise} := (z_{t})_{t \in \mathcal{T}} \in \mathbb{R}^{|\mathcal{T}|}$ and matrix $\gls{subsetfeature} := (x_{t})_{t \in \mathcal{T}} \in \mathbb{R}^{|\mathcal{T}| \times d }$ so that $Y_{\mathcal{T}} = X_{\mathcal{T}} \theta + Z_{\mathcal{T}}$. 
When $X_{\mathcal{T}}$ has full rank we define $\gls{Leastsquare}(X_{\cT}, Y_{\cT}) = (X_{\cT}^\top X_{\cT})^{-1} X_{\cT}^\top Y_{\cT}  $ the (unregularized) ordinary least-squares estimate of $\theta$ when using only observations at times $t \in \mathcal{T}$. 
For $i \in \mathbb{N}$ we define $\gls{ni} = d 2^{i-1}$, $\gls{Ti} = d (2^{i} - 1)$ and $\gls{deltai} = \min(d n_i/T, 1)$.

\subsection{Algorithm description}

Consider the centered case $c=0$. The pseudocode of~\nous~is presented as Algorithm~\ref{alg:main-alg} and is composed of three phases, whose names are inspired by the "explore-then-commit" strategy of~\citet{garivier2016explore}. The main addition is the first phase that allows to determine how many exploration rounds are necessary, as those must depend on $\|\theta\|_{A}$, as shown by our analysis.

\textbf{Phase 1}: ("warmup phase", lines 1-5): We sample actions in $(A^{1/2} e_j)_{j \in [d]}$ in a round-robin fashion. 
At times $t \in (T_i)_{i \in \mathbb{N}}$ with we compute $\gls{hattheta_i} := \LS(X_{i}, Y_{i})$ the least-squares estimate of $\theta$ using the observations during time interval $[T_i,T_{i+1}-1]$ and if $\|\hat{\theta}_{i} \|_{A}$ exceeds a threshold equal to $\alpha U(\delta_i,n_i)$ where $\gls{alpha} \in \RR^+$ is a positive parameter of the algorithm and with 
\begin{equation} \label{eq:Udeltan-def}
	\gls{chiconf}^2 = \frac{\sigma^2 d^2}{n_i} \bigg( 1  + 
    2\sqrt{\frac{1}{d} \log \frac{1}{\delta_i}} 
    + \frac{2}{d} \log \frac{1}{\delta_i}\bigg) \,. 
\end{equation}
We exit this phase that we number \gls{hati}, otherwise we keep going. 
When exiting, we compute $\gls{hatB} := \|\hat{\theta}_{\hati} \|_{A} $ the norm estimate. 
The threshold $U$ is designed to estimate the magnitude of $\theta$ correctly, so that with high probability $c_1 \|\theta\|_{A} \le \hat{B} \le c_2 \|\theta\|_{A}$, with $c_1,c_2$ well-chosen universal constants. 

\textbf{Phase 2}: ("exploration phase", lines 6-9): We sample actions in $(A^{1/2} e_j)_{j \in [d]}$ in a round-robin fashion during an amount of time $\gls{Nexplore} := d \sigma \big\lceil \sqrt{T} / \hat{B} \big\rceil$.
The goal of this phase is simply to gather enough samples to estimate $\theta$ accurately, to be able to find a good action by the end of the phase with high probability.
The analysis shows that the number of samples needs to depend on $\|\theta\|_{A}$, which is why we use the estimate $\hat{ B  }$ from the previous phase.

\textbf{Phase 3}: ("commit phase", lines 10-12): We compute $\gls{hattheta} := \LS(X, Y)$ the least-squares estimate of $\theta$ using the past observations gathered during phase 2, and play greedily by selecting decision $\arg\max_{x \in \mathcal{X}} x^\top \hat{\theta}$ until the time horizon runs out. 
The goal of this phase is simply to maximize reward, as we are certain that enough samples have been gathered and this decision should be close to the optimal one $x^\star(\theta)$ with high probability. 

\begin{algorithm} 
  \SetKwInOut{Input}{Input}
  \Input{Ellipsoid matrix $A$, noise level $\sigma$, time horizon $T$, parameter $\alpha \ge 0$.}
  \KwInit{$i = 0$.} 
  \While{warm-up}{ 
    \For{$t = T_i$ to $T_{i+1}-1$}{ 
      Play $x_t = A^{1/2} e_{(t \, \mathrm{ mod } \, d) + 1}$.
      Receive $y_t$. 
    } 
    Set $X_i = x_{[T_i, T_{i+1}-1]}$ 
    and $Y_i =y_{[T_i, T_{i+1}-1]} $.  
    Set $\hat \theta_i = \LS(X_i, Y_i)$.
    Update $i \leftarrow i + 1$. \\
    Exit warm-up if $\|\hat \theta_i\|_A > \alpha U(\delta_i, n_i)$,\, defined in \eqref{eq:Udeltan-def}
  }
    Set $\hat B = \| \hat \theta_i\|_A$ and $N_e= d \sigma \big\lceil \sqrt{T} / \hat{B} \big\rceil$ \\
    \For{$t= T_i$ to $T_i + N_e- 1$}
    {Play $x_t = A^{1/2} e_{(t \,\mathrm{ mod } \, d) + 1}$. Receive $y_t$.} 
    Set $X = x_{[T_i, T_i + N_e- 1]}$ 
    and $Y =y_{[T_i, T_i + N_e- 1]} $.  \\
    Set $\hat \theta = LS(X, Y)$.
    \\
    \For{$t= T_i + N_e(\hat B)$ to $T$}
      {Play $x_t = x^\star(\hat \theta)$. Receive $y_t$.} 
  \caption{\nous~ (\algoname{Explore-Explore-Then-Commit})}\label{alg:main-alg} 
\end{algorithm}

\begin{remark}[Computational complexity]
A run of~\nous~ takes time $O(dT + d^2\log(T/d) + d^3)$ and memory $O(d^2)$, since it involves at most $\log_2(T/d) + 1$ least-squares estimates of $\theta$, and all those estimations require inverting $d \times d$ matrices which are all proportional to $A$, so it suffices to invert $A$ once. 
\end{remark}

\subsection{Regret upper bound}
We now upper bound the regret of~\nous, in the centered case $c=0$.
We present the main elements of analysis here, and present complete proofs in appendix. 

\paragraph{Regret of phases 2 and 3}
Assume that phase $1$ was successful. Denote by \gls{Econdiexplor} the expectation conditionally on all rounds before phase 3.
The regret caused by a round of phase $2$ is at most $2 \|\theta\|_A$ by the Cauchy-Schwarz inequality. 
The regret caused by a round of phase $3$ is $ \theta^\top(x^\star(\theta) - x^\star(\hat \theta))$ we always select the same action. 
The regret caused by phases 2 and 3 is hence upper bounded by 
\[
  2 N_e \| \theta\|_A
  + T\, \E_\expl \big[ \theta^\top \big(x^\star(\theta) - x^\star(\hat \theta)\big) \big] 
  \,.
\]
We note that, during phases 1 and 2, we play uniformly along the axes of the ellipsoid, which ensures that the design matrix \gls{design} stays proportional to $A$
(after a number of exploration rounds multiple of $d$, which is always the case at the end of exploration in our algorithm). 

As observed by \citet{rusmi_linearly_2010}, when the action set has positive curvature,
the commit error scales \emph{quadratically} with the mean-squared error of estimation
(see Lemma~\ref{lem:ub_inst_regret_square} in Appendix~\ref{app:proofs_alg}),
which itself scales with $ 1 / N_e$; this is crucial to ensure that the regret
grows proportionally to $\sqrt T$.
\begin{lemma}
  \label{lem:commit_error}
  Let $\hat \theta$ the least-squares estimator after $N_e$ rounds of exploration, then 
  conditionally on the exploration rounds,
    \[
    \E_{\expl} \big[\theta^\top \big(x^\star(\theta) - x^\star(\hat \theta)\big)\big]
    \leq 
    \frac{d^2 \sigma^2}{\|\theta\|_A N_e} 
    + 
        2 \|\theta \|_A e^{(2d / 3 - N_e \|\theta\|^2 / (3 \sigma^2 d))^-}
    \,.
    \]
\end{lemma}
See Appendix~\ref{app:proof_main_thm} for a proof. The regret incurred by phases 2 and 3 is hence upper bounded as 
\[
  2 N_e \| \theta \|_A
  + T \frac{d^2 \sigma^2}{\|\theta\|_A N_e}  + T 
  \bigg(
  \frac{d^2 \sigma^2}{\|\theta\|_A N_e} 
  + 
  2 \|\theta \|_A e^{(2d / 3 - N_e \|\theta\|^2 / (3 \sigma^2 d))^-}
  \bigg)
  \,.
\]
Our analysis therefore suggests selecting $N_e \approx \sigma d \sqrt T / \|\theta\|_A$ to balance the first two terms in the right-hand side.  
Doing this would yield a regret upper bound of $O(d \sqrt{T})$ which is the correct minimax rate. 
Of course, since $\|\theta\|_A$ is initially unknown, we must estimate it during phase 1. 

\paragraph{Regret of phase 1}

Lemma~\ref{lem:warm-up} shows that the regret caused by phase 1 is small enough to ensure minimax optimality, and that the value of $\hat{B}$ output at the end of phase 1 is a good enough estimate of $\|\theta\|_A$ with high probability. 
This lemma also justifies the choice of the stopping criterion for phase 1.

  \begin{lemma}[Regret of Phase 1]
    \label{lem:warm-up}
    Set $\alpha = 3$ and let $\hat B$ be the output of phase 1. We have
      \[
          \P\Big[
              \hat B  \notin \big[  (1 / 2)\|\theta\|_A, (3 / 2)\|\theta\|_A\big]
              \Big]
          \leq \frac{164 \, \sigma^2 d^2}{T \|\theta\|_A^2}  
          \overline \log \bigg(\frac{T\|\theta\|_A^2}{d^2 \sigma^2}\bigg)  
          + 48 \frac{d}{T} \,, 
      \]
      and the expected length of phase 1 is at most 
      \[
        \E[N_{\warmup}]
          \leq \frac{164 \, \sigma^2 d^2}{\|\theta\|_A^2}  
          \overline \log \bigg(\frac{T\|\theta\|_A^2}{\sigma^2 d^2}\bigg) + 48d
          \,.
      \]
  \end{lemma}

\paragraph{On the design of phase 1}  
The threshold $U(\delta_i,n_i)$ is chosen such that the probability of stopping too
early at time $T_i$ is less than $\delta_i$.
So $U(\delta_i,n_i)$ and $\delta_i$ can be interpreted as an upper confidence bound,
and a failure probability respectively.
The exploration design is chosen so that $X_i^\top X_i = (n_i / d) A$ after a phase of
length $n_i$, which guarantees at the end of subphase $i$ of the warm-up, the
probability of failure is less than $\delta_i$:
\[
  \P \big[
    \| \hat \theta_i - \theta\|_A^2 \geq U(n_i, \delta_i)^2 
    \big]
    \leq \delta_i \,.
\]
Interestingly, our analysis shows that $\delta_i$ must (counterintuitively) be an increasing function of $i$ to yield a minimax optimal algorithm. 
The cornerstone of both the design of phase 1 and its analysis is a concentration
inequality for least-squares regression, by \citet{hsu_tail_2011}, which is an
extension of the tail bounds for chi-squared variables by \citet{laurent2000} to
sub-gaussian noise.

\begin{lemma}[Concentration of error in linear regression ]
  \label{lem:chi_squared_concentration}
  In linear regression with $Y = X \theta + Z$ and $\hat \theta = \LS(X, Y)$, assume that
  all entries of $Z$ are $\sigma^2$-subgaussian and independent, and that $X$ is full
  rank. Then for all $x \ge 0$
   \[
    \P\big[\|\hat \theta - \theta\|_{X^\top X}^2 \geq 
      \sigma^2(d + 2\sqrt{d x} + 2x)
    \big] 
    \leq e^{-x}
    \,.
  \]
\end{lemma}

Putting the three phases together, we can now state the regret upper bound of~\nous in Theorem~\ref{thm:full_regret_ETC}, and the detailed proof is presented in appendix. 
Corollary~\ref{cor:etc-is-minimax} shows that the regret upper bound matches our regret lower bound of Theorem~\ref{theorem:lower_bound} up to a universal multiplicative constant, so that~\nous~is locally asymptotically minimax optimal as announced. 

\begin{theorem}
  \label{thm:full_regret_ETC}
  Consider $\mathcal{X}= \{ x \in \mathbb{R}^d: \|x\|_{A^{-1}} \le 1 \}$ a centered ellipsoid. 
  The regret of~\nous~tuned with $\alpha=3$ admits the upper bound for any $\theta \in \mathbb{R}^d$ and any $T \ge 0$: 
  \[
    R_T(\theta) \leq 6 d \sigma \sqrt T  
    + \frac{984 \, \sigma^2 d^2}{\|\theta\|_A}
    \overline \log \bigg( \frac{T \|\theta\|_A^2}{\sigma^2 d^2}\bigg) 
    + 
    290 \, d \|\theta\|_A
    + 
    2 T \|\theta \|_A e^{(2d / 3 - (2/9) \sqrt T \|\theta\|_A / \sigma)^-}
    \,.
  \]
\end{theorem}
\begin{corollary}
  \label{cor:etc-is-minimax}
  \nous~is locally asymptotically minimax optimal: for any  $\theta \in \R^d$, 
  \[
    R_T(\theta) =
    \mathcal O \Big( 
      \min \big( 
         \sigma d \sqrt T +  d \|\theta\|_A,  T\|\theta\|_A 
      \big)
      \Big)
    \,.
  \]
\end{corollary}
\begin{proof}\textbf{of Corollary~\ref{cor:etc-is-minimax}}
 We proceed by considering multiple cases, depending on the value of $T$.
 If $T \leq 36 d^2 \sigma^2 / \|\theta\|_A^2$, then 
 $
 2 T\|\theta\|_A \leq 12 \sigma d \sqrt T 
 $.
 So by the Cauchy-Schwarz bound, the regret is less than
 $
 R_T(\theta) \leq \min \big(2 T \|\theta\|_A,  12 \sigma d \sqrt T\big)
 $.
 Otherwise,
 \[
 2d / 3 - (2/9) \sqrt T \|\theta\|_A / \sigma 
 \leq - (1 / 9)\sqrt T \|\theta\|_A / \sigma
 \,,  
 \]
 and the final term in the regret bound is upper bounded by
 \begin{multline*}
     T \|\theta \|_A e^{(2d / 3 - (2/9) \sqrt T \|\theta\|_A / \sigma)^-}
   \leq 
     T \|\theta \|_A e^{-(1/9) \sqrt T \|\theta\|_A / \sigma} 
    =  \sigma  \sqrt T \bigg(
     \frac{\sqrt T \|\theta \|_A}{\sigma}  e^{-(1/9) \sqrt T \|\theta\|_A / \sigma}
    \bigg)  \\
    \leq 9 \sigma  \sqrt T \sup_{X \geq 2 d / 3} \big(
      X e^{-X}
    \big)
    \leq 6 e^{-2d / 3} \sigma d \sqrt T 
   \,.
 \end{multline*}
 Similarly, for the first logarithmic term, 
 \begin{equation*}
     \frac{\sigma^2 d^2}{\|\theta\|_A}
    \overline \log \bigg( \frac{T \|\theta\|_A^2}{\sigma^2 d^2}\bigg) 
    \leq \sigma d \sqrt T 
     \frac{\sigma d}{\|\theta\|_A \sqrt T}
    \overline \log \bigg( \frac{T \|\theta\|_A^2}{\sigma^2 d^2}\bigg) 
    \leq \sigma d \sqrt T
    \sup_{X \geq 1 / 6} \frac{1}{X} \overline \log (X^2)
     = 6 \sigma d \sqrt T 
     \,.
 \end{equation*}
 Therefore, the total bound is at most $O(d \sigma \sqrt T + d\|\theta\|)$ as announced.
\end{proof}

\subsection{Non-centered Action Sets}
\label{sec:non-centered}
Now consider the general, non-centered case $c \ne 0$.
In Appendix~\ref{app:reduction}, we describe a general procedure to extend~\nous~(and any other algorithm with similar structure) designed for the centered case to the non-centered case. 
Theorem~\ref{thm:noncentered} shows that~\nous~is also locally asymptotically minimax
optimal in the non-centered case; it is a direct consequence of
Theorem~\ref{thm:non_centered_gen} in Appendix~\ref{app:reduction}
\begin{theorem}\label{thm:noncentered} 
  Consider $\mathcal{X}= \{ x \in \mathbb{R}^d: \|x-c\|_{A^{-1}} \le 1 \}$ an arbitrary  ellipsoid.
  Set $\alpha = 3$. 
  The regret of~\nous~with Reduction~\ref{alg:non-centered} admits the upper bound for any $\theta \in \mathbb{R}^d$ and any $T \ge 0$: 
  \begin{equation*}
    R_T(\theta)
    \leq 7 d\sigma \sqrt T
    + \frac{2622 \, \sigma^2 d^2}{\|\theta\|_A}
    \overline \log \bigg(\frac{T\|\theta\|_A^2}{4 \sigma^2 d^2}\bigg)
    +
    2 T \|\theta \|_A e^{(2d / 3 - (1/9) \sqrt T \|\theta\|_A / \sigma)^-}
    + 392 \, d \|\theta\|_A
    \,.
  \end{equation*}
\end{theorem}

\section{Numerical Experiments}

We now compare~\nous~to state-of-the-art algorithms \algoname{OFUL} and \algoname{OLSOFUL} using numerical experiments. 
Regret of algorithms is averaged over $20$ independent runs, and both averages and confidence intervals are presented.
The source code for those numerical experiments will be released publicly after review.
We denote by $\nous(\alpha)$ the~\nous~algorithm with parameter $\alpha$, and we consider both \nous(1) and \nous(3). 
We consider action set $\mathcal{X}$ as the unit ball and a time horizon of $T=10^4$. 
The confidence ellipsoids in \algoname{OFUL} and \algoname{OLSOFUL} require a confidence parameter $\delta$, which can be interpreted as a probability of failure, and we set it to $\delta =1/T$, to ensure that the regret caused by failure does not scale linearly in $T$.

We use a commercially available optimization solver \cite{gurobi} to handle the bilinear maximization problem $\max_{(x,\theta)} x^\top \theta$ subject to $(x,\theta) \in \mathcal{X} \times \mathcal{C}_t$ with $\mathcal{C}_t$ the confidence ellipsoid at time $t$, which must be solved at each step to implement \algoname{OFUL} and \algoname{OLSOFUL}. 
We limit Gurobi's available time to $1$ second per time step, which already amounts to $2.7$ hours per run, that is $2.5$ days to average the regret over $20$ independent runs.
Gurobi is allowed to use all $20$ cores of an i9-129000H processor, and outputs the best solution found within the allotted time, with an early stop if the estimated relative gap to optimality is less than $1\%$.

\textbf{Regret $R_T(\theta)$ vs norm of $\theta$ \quad}
We compare in Figure~\ref{fig:norm_comparison} the final regret of the algorithms for different values of $\|\theta\|_2$ chosen as  $\|\theta\|_{2} \in \{ d/\sqrt{T}, 0.1, 1, 10, 25, 50\}$. \algoname{OFUL} requires an input parameter $S$ which should be an upper bound for $\|\theta\|$. We set $S$ as half the range norms so that $S = 25$. 
Due to the enormous computation complexity of optimistic algorithms, we limit ourselves to a small dimension $d=3$: a larger dimension would already require over a week of computation to generate meaningful results.

\textbf{Regret $R_T(\theta)$ vs dimension $d$ \quad} 
We compare in Figure~\ref{fig:dimension_comparison} the final regret of the algorithms for different values of $d$, and we choose $\theta$ with a fixed norm $\|\theta\|_{2} = 10$.
We consider dimensions $2 \le d \le 90$. 
We could not make the optimistic algorithms work in reasonable time for $d > 4$, so for high dimensions, solely~\nous~is presented. 
As suggested by our theoretical results, the regret of~\nous~seems to scale linearly with the dimension.

\textbf{Runtime \quad} We compare the runtime (CPU time) of~\nous~and \algoname{OFUL} for different norms $\|\theta\|$, where dimension is $d=3$ in Fig.\ref{fig:norm_runtime} and for different dimensions $d$ and fixed norm $10$ in Fig.\ref{fig:dimension_runtime}. 
Runtime includes the time to compute the action to be played, the time used to solve the bilinear optimization problem, and the time to update the estimated parameters $\hat{\theta}_t$ using the Sherman Morrison algorithm for the inversion of the design matrix.
We present the ratio between the runtime of \algoname{OFUL} and that of~\nous~, and observe that \algoname{OFUL} is slower by about $5$ orders of magnitude, an enormous difference. This clearly illustrates the lower computational complexity of~\nous~with respect to optimistic algorithms. 
Also, for small norms in dimension $3$ and for dimension greater than $4$, it is harder for Gurobi to find a good solution, and it maxes out its allotted time.

\begin{figure}
    \centering    
    \subfloat[Large norms]{\includegraphics[width=0.49\textwidth]{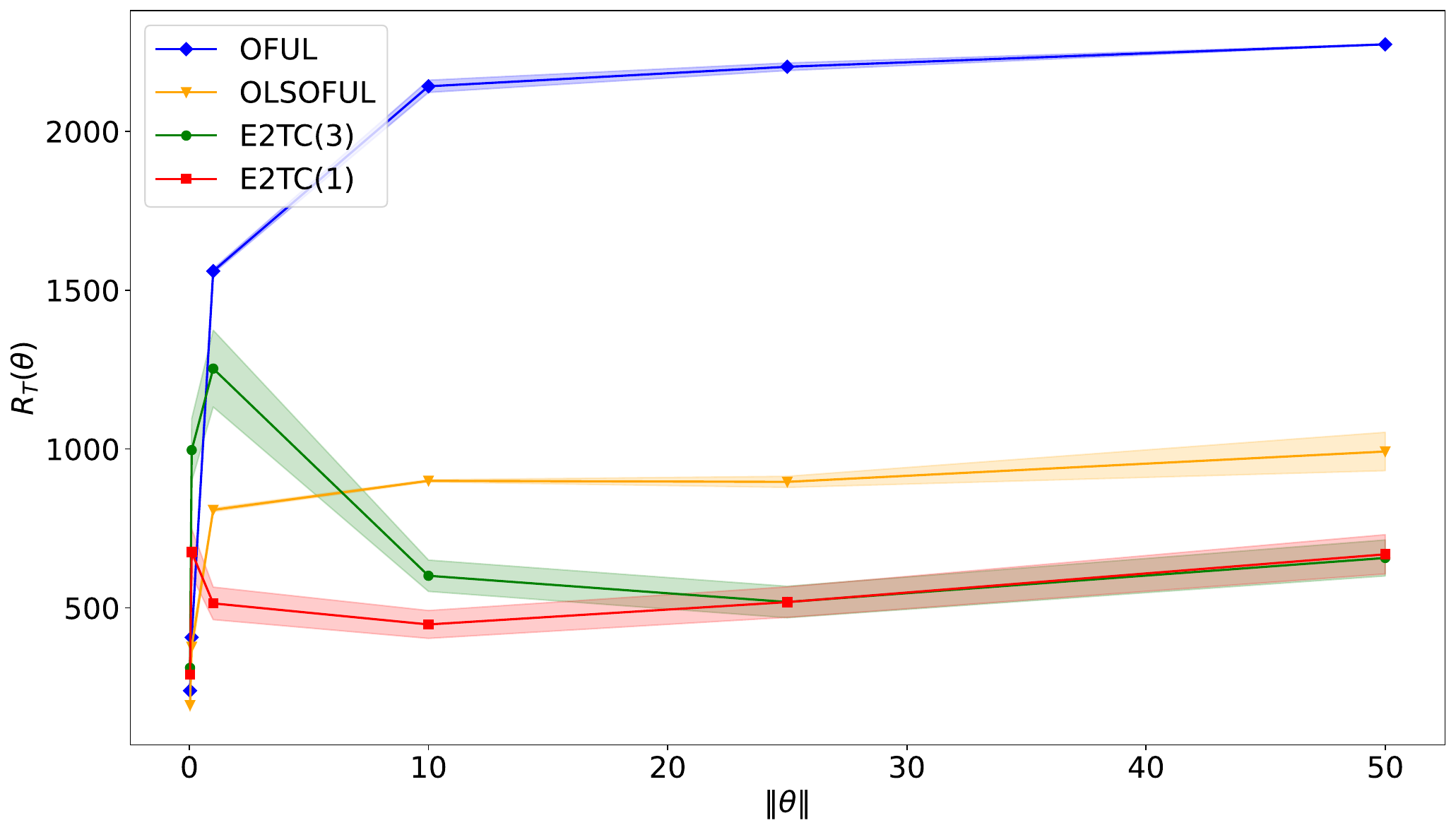}}
    \subfloat[Small norms]{\includegraphics[width=0.49\textwidth]{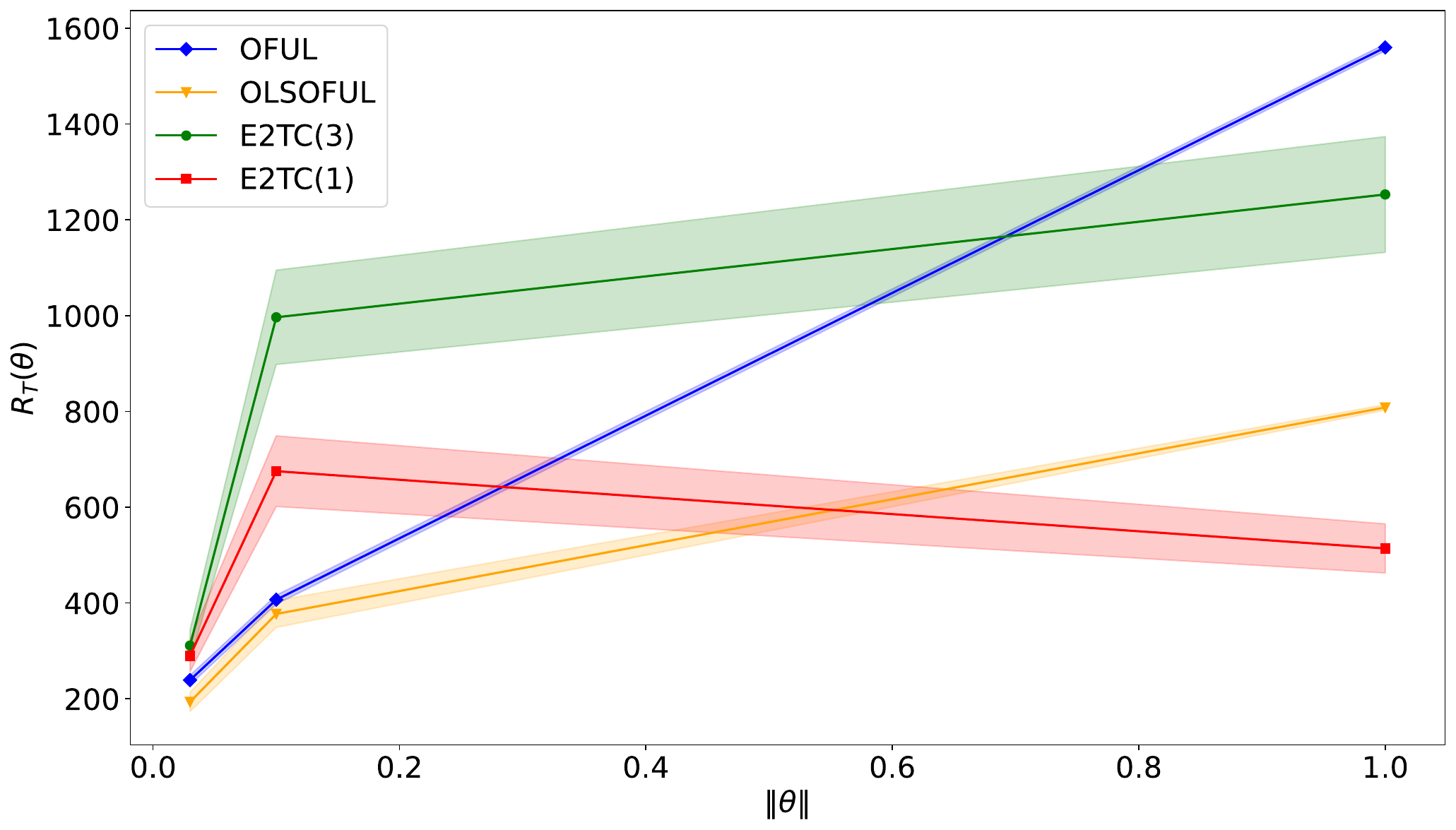}} 
    \caption{Comparison of the final regret of~\nous, \algoname{OFUL} and \algoname{OLSOFUL} for different norms.}
    \label{fig:norm_comparison}
\end{figure}

\begin{figure}[h]
    \centering    
    \subfloat[Large dimensions]{\includegraphics[width=0.49\textwidth]{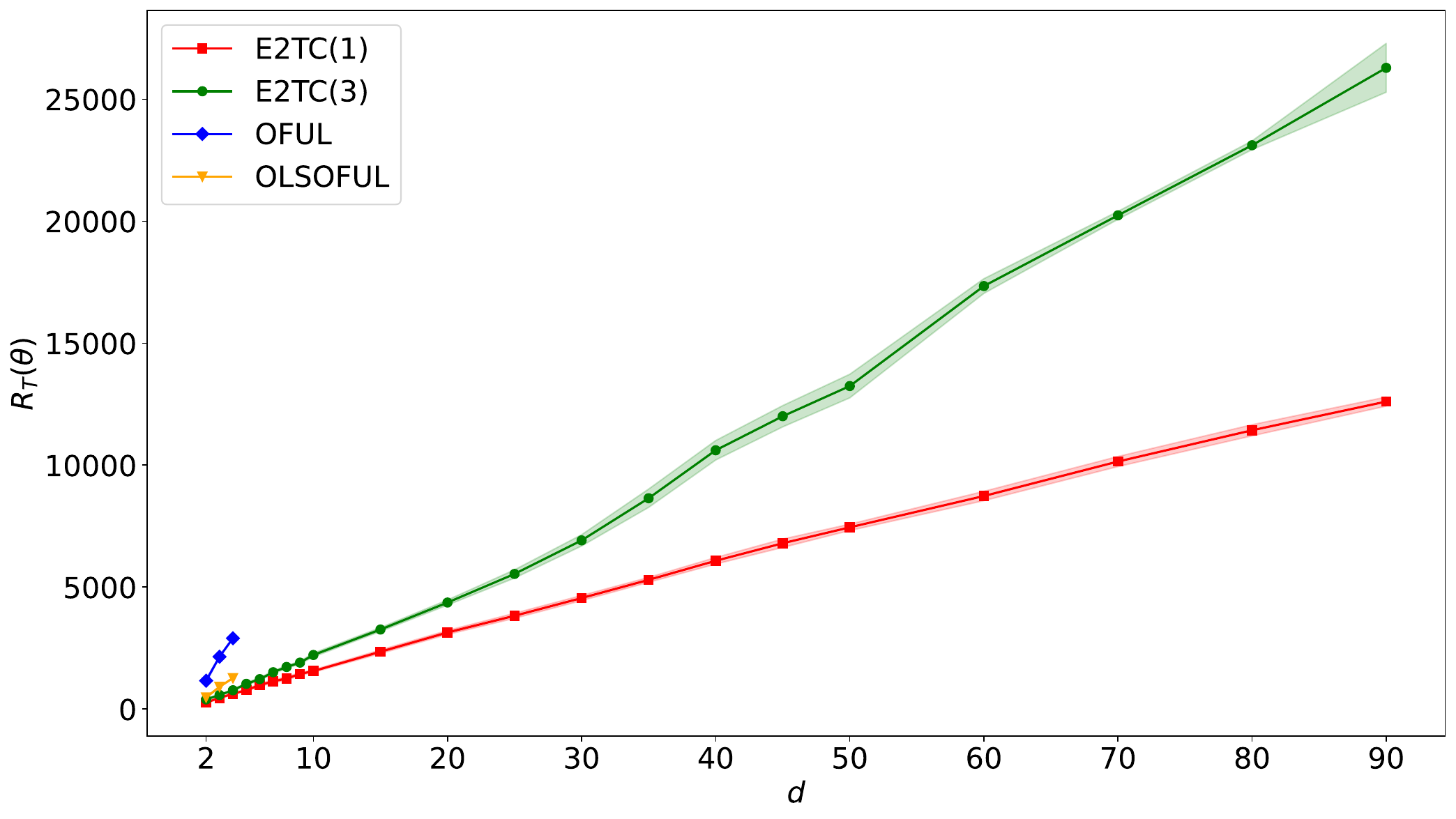}}
    \subfloat[Small dimensions]{\includegraphics[width=0.49\textwidth]{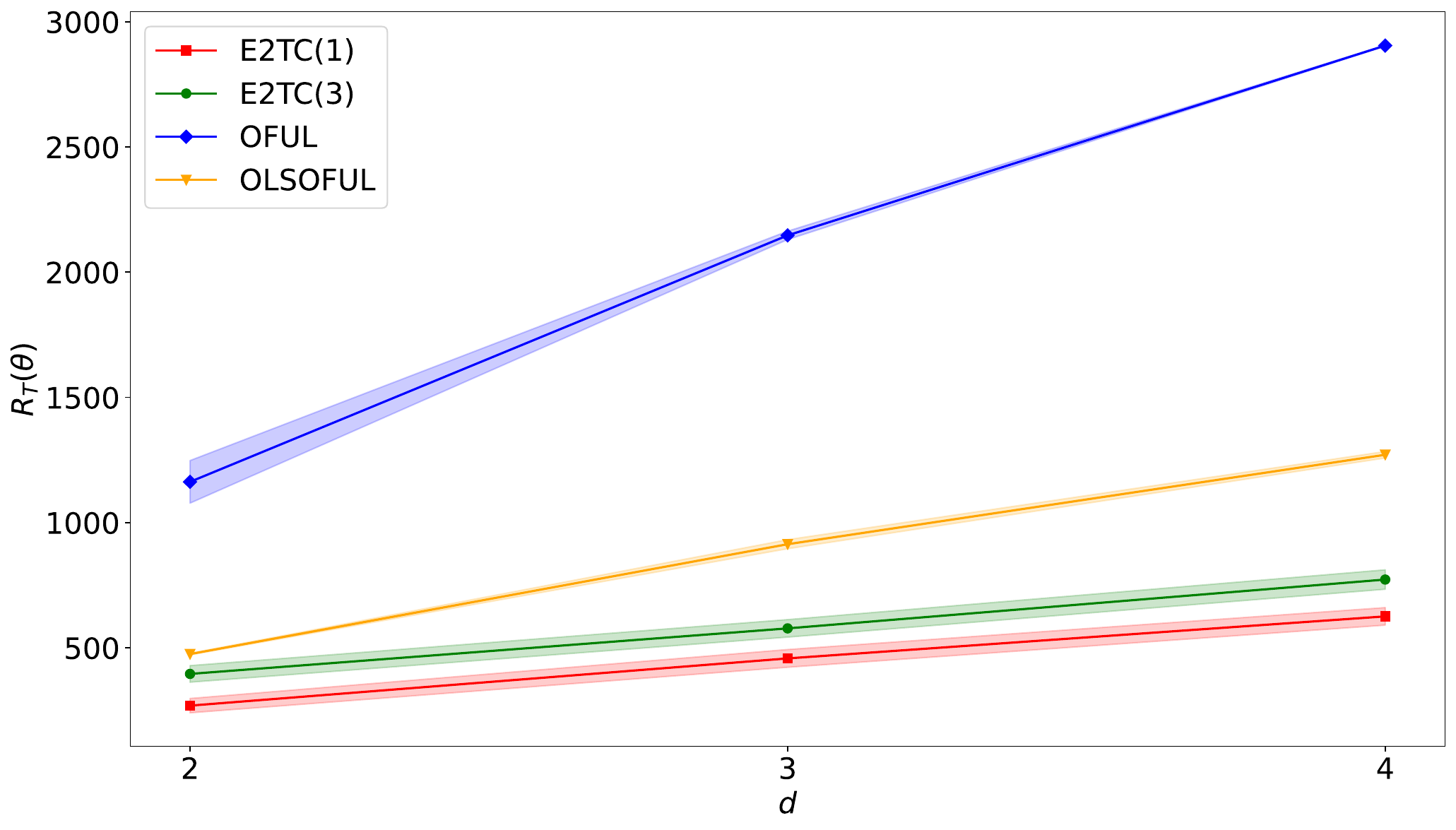}} 
    \caption{Final regret of~\nous, \algoname{OFUL} and \algoname{OLSOFUL} as a function of the dimension $d$.}
    \label{fig:dimension_comparison}
\end{figure}

\begin{figure}[h!]
    \centering
            
    \footnotesize
\subfloat[Norm runtime comparison]
        {\begin{tabular}{|| c | c | c | c | c ||} 
         \hline
         $\|\theta\|$  & 0.1 & 1 & 10 & 50 \\ [0.5ex] 
         \hline\hline
         Runtime Ratio &  $6.1 \times 10^4$ & $1.0\times10^5$ & $1.7\times10^5$ & $8.7 \times 10^4$ \\ [1ex] 
         \hline
        \end{tabular}
        \label{fig:norm_runtime}
        }
    \subfloat[Dimension runtime comparison]
        {\begin{tabular}{||c | c | c | c||} 
         \hline
         $d$ & 2 & 3 & 4 \\ [0.5ex] 
         \hline\hline
         Runtime Ratio & $1.2 \times 10^4$ & $1.7\times 10^5$ & $6.7 \times 10^5$  \\ [1ex] 
         \hline
        \end{tabular}
        \label{fig:dimension_runtime}
        }
     \caption{Runtime comparison of~\nous~and \algoname{OFUL} for different norms and dimensions.}
    \label{fig:runtime}
\end{figure}

\section{Conclusion}\label{sec:Conclusion}
We have proposed the first known provably minimax optimal algorithm for linear stochastic bandits over ellipsoids called~\nous. 
We derived a new information-theoretic lower bound on the regret of any algorithm equal to $\Omega(\min(d \sigma \sqrt{T}, \|\theta\|_{A} T))$, and showed that the regret of~\nous~matches this bound so that it is minimax optimal. 
The~\nous~algorithm is locally asymptotically minimax optimal, a stronger optimality notion than classical minimax optimality.
Also~\nous~is highly computationally efficient, requiring time $O(dT + d^2\log(T/d) + d^3)$ and memory $O(d^2)$, unlike optimistic algorithms for which, to the best of our knowledge, no polynomial time implementation exists, even when the action set is a unit ball.
Numerical experiments confirm our theoretical findings both in terms of regret as well as running time. 
An interesting direction for future research is whether or not the general structure of ~\nous~can be adapted to solve other structured bandit problems. 

\section*{Acknowledgements}

This paper is supported by the CHIST-ERA Wireless AI 2022 call MLDR project (ANR-23-CHR4-0005), partially funded by AEI and NCN under projects PCI2023-145958-2 and 2023/05/Y/ST7/00004, respectively.

\bibliographystyle{plainnat}
\bibliography{linbandits} 

\newpage 
\appendix

\section{Proofs for Section~\ref{sec:alg}}
\label{app:proofs_alg}
\subsection{Proof of Theorem~\ref{thm:full_regret_ETC}}
\label{app:proof_main_thm}

\begin{proof}
  We use the natural template for the analysis of Explore-Then-Commit (ETC) algorithms.
  The first step is to use the Cauchy-Schwarz inequality to bound the instantaneous
  regret $\theta^\top(x^\star(\theta) - x_t) \leq 2\|\theta\|_A$ on all the exploration
  rounds:
  \[
    R_T(\theta)
    \leq 
    2 \|\theta\|_A \E[N_\warmup + N_{e}(\hat B)]
    + T\, \E\big[\theta^\top (x^\star(\theta) - x^\star(\hat \theta))\big]
    \,.
  \]
  We shall then control separately the expected total number of exploration rounds,  
  \begin{equation}
    \label{eq:total_exploration_bound}
    \E[N_\warmup + N_{e}(\hat B)]
     \leq \frac{2d\sigma \sqrt T}{\|\theta\|_A} 
      + \frac{328 \, \sigma^2 d^2}{\|\theta\|_A^2}  
         \overline \log \bigg(\frac{T\|\theta\|_A^2}{\sigma^2 d^2}\bigg)  + 97 d 
      \, ,
  \end{equation}
  and the expected cost of the commit rounds, 
  \begin{equation}
    \label{eq:commit_error_final}
    \E\big[\theta^\top (x^\star(\theta) - x^\star(\hat \theta))\big]
    \leq 
    \frac{3 d \sigma}{2\sqrt T} 
    +
    2 \|\theta \|_A e^{(2d / 3 - (2/9) \sqrt T \|\theta\|_A / \sigma)^-}
    +
    \frac{328 \, \sigma^2 d^2}{T \|\theta\|_A}  
      \overline \log \bigg(\frac{T\|\theta\|_A^2}{\sigma^2 d^2}\bigg)  
      + 96 \frac{d \|\theta\|_A}{T} 
    \,.
  \end{equation}
  The three bounds combined yield the claimed result. The first step in the proof of both
  of these bounds is to control the warm-up behavior. This is done via
  Lemma~\ref{lem:warm-up} stated above and proved below.
  It then suffices to analyze the algorithm on the event that the warm-up was successful.
  We thus define the corresponding good event $\cG$ as
  \begin{equation*}
    \cG = \bigg\{
      \frac{\|\theta\|_A}{2} \leq \hat B \leq \frac{3\|\theta\|_A}{2}
  \bigg\} \,.
  \end{equation*}

\paragraph{Exploration rounds}
On the good event $\cG$, the estimate $\hat B$ is at least $\|\theta\| / 2$, so 
\[
 N_{e}(\hat B) \mathds{1} \{ \cG \}
 = d \bigg\lceil\frac{\sigma \sqrt T}{\hat B} \bigg\rceil \mathds{1} \{ \cG \}
 \leq d \bigg\lceil\frac{2\sigma \sqrt T}{\|\theta\|_A} \bigg\rceil 
 \leq d \frac{2\sigma \sqrt T}{\|\theta\|_A} + d
 \,.
\]
And therefore, by Lemma~\ref{lem:warm-up}, in expectation, bounding $N_e(\hat B)$
by the worst-case $T$ on $\bar \cG$, 
\[
 \E[N_e(\hat B)]
 \leq d \frac{2\sigma \sqrt T}{\|\theta\|_A} + d
 + T \P\big[\, \overline \cG \, \big]
 \leq d \frac{2\sigma \sqrt T}{\|\theta\|_A} + d
  +  \frac{164 \, \sigma^2 d^2}{\|\theta\|_A^2}  
         \overline \log \bigg(\frac{T\|\theta\|_A^2}{ \sigma^2 d^2}\bigg)  + 48 d 
  \,.
\]

\paragraph{Commit error}

We isolate the guarantees for the least-squares estimator in
Lemma~\ref{lem:commit_error}, which we restate and prove below.
Given this lemma, by the tower rule, and since $\cG$ only depends on the exploration
rounds (thus is measurable with respect to the $\sigma$-algebra generated by actions and
observations in the exploration rounds), 
\begin{multline*}
     \E[
        \theta^\top (x^\star(\theta) - x^\star(u))
        \mathds{1} \{ \cG \}
    ]
    = \E[ \E_\expl[
        \theta^\top (x^\star(\theta) - x^\star(u))
        ]
        \mathds{1} \{ \cG \}
      ] \\
    \leq
    \E \bigg[\bigg( \frac{d^2 \sigma^2}{\|\theta\|_A N_e} 
      +
    2 \|\theta \|_A e^{(2d / 3 - N_e \|\theta\|^2 / (3 \sigma^2 d))^-}
  \bigg)
        \mathds{1} \{ \cG \}
    \bigg]  
    \leq 
    \frac{3 d \sigma}{2\sqrt T} 
    + 
    2 \|\theta \|_A e^{(2d / 3 - (2/9) \sqrt T \|\theta\|_A / \sigma)^-}
  \,, 
\end{multline*}
where we used the fact that $\hat B \leq 3  \|\theta\|_A / 2$ on the event $\cG$, and thus that $N_e \geq d \sigma \sqrt T / (3 \|\theta\|_A / 2)$
\[
    \E[
      \theta^\top (x^\star(\theta) - x^\star(u))
    ]
    \leq
     \E[
        \theta^\top (x^\star(\theta) - x^\star(u))
        \mathds{1} \{ \cG \}
    ]
    + 2\|\theta\|_A \P\big[\, \overline \cG\, \big] 
    \,.
\]
The claimed bound follows by applying the previously shown inequalities.
\end{proof}

\begin{lemma*}\!\emph{\ref{lem:commit_error}} \;
  If $\hat \theta$ is the least-squares estimator after $N_e$ rounds of exploration, then 
  conditionally on the exploration rounds,
    \[
    \E_{\expl} \big[\theta^\top \big(x^\star(\theta) - x^\star(\hat \theta)\big)\big]
    \leq 
    \frac{d^2 \sigma^2}{\|\theta\|_A N_e} 
    + 
        2 \|\theta \|_A e^{(2d / 3 - N_e \|\theta\|^2 / (3 \sigma^2 d))^-}
    \,.
    \]
\end{lemma*}
\begin{proof}
  To ease notation, we remove the subscript on $\E_{\expl}$ inside this proof; all
  expectations and probability are taken conditionally on the randomness of the
  exploration rounds.
  At the end of the exploration phase, consider the matrices $Y \in \R^{N_e}$, and $X \in
  \R^{N_e \times d} $, defined by stacking the exploratory actions and responses so that
  \[ 
    Y = X \theta + Z
    \quad \text{ and } \quad
    \hat \theta = (X^\top X)^{-1}X^\top Y
    \,. 
  \]
  Let us start by evacuating very bad cases, in which the estimator is completely off. 
  Depending on whether $\theta^\top A \hat \theta \leq 0$, we upper bound the
  instantaneous regret by either $2\|\theta\|_A$, or using Lemma~\ref{lem:ub_inst_regret_square},
  \begin{multline}
    \label{eq:inst_reg_very_bad}
    \E \big[\theta^\top \big(x^\star(\theta) - x^\star(\hat \theta)\big)\big]
    \leq 
    2 \|\theta\|_A \P\big[\theta^\top A \hat \theta \leq 0\big]
    + 
    \frac{1}{\|\theta\|_A}\E \big[\|\theta - \hat \theta \|_A^2\big] \\
    \leq 
    2 \|\theta\|_A \P[\| \hat \theta - \theta \|_A \geq \|\theta\|_A]
    + 
    \frac{1}{\|\theta\|_A}\E \big[\|\theta - \hat \theta \|_A^2\big]
    \,.
  \end{multline}
  The second term is essentially the mean-squared error of the least-squares estimator
  \[ 
    \E [\|\theta - \hat \theta \|_A^2]
    = \E [ \|(X^\top X)^{-1} X^\top Z \|_A^2]
    = \E [ \|A^{1/2}(X^\top X)^{-1} X^\top Z \|^2]
  \]
  Now note that for any matrix $M \in \R^{N_e \times d}$, since the covariance matrix of
  $Z$ is diagonal with entries in $[0,\sigma^2]$, we have
  \[ 
    \E[\|M Z \|^2]
    = \E [\Tr( Z Z^\top M^\top M) ] 
    =  \sum_{t \in \Texpl} \E[z_t^2] (M^\top M)_{t, t}
    \leq \sigma^2 \Tr(M^\top M) \,.
  \]
  In particular, taking $M = A^{1/2}(X^\top X)^{-1} X^\top$,  
  \[
    \E [ \|A(X^\top X)^{-1} X^\top Z \|^2]
    \leq \sigma^2 \Tr\big(A (X^\top X)^{-1}\big) 
    = \frac{d^2\sigma^2}{N_e}  
    \,, 
  \]
  where we used the fact that $X^\top X = (N_e / d) A$ to obtain the final equality. 

  To treat the first term in \eqref{eq:inst_reg_very_bad}, we use again
  the sub-gaussian concentration, Lemma~\ref{lem:chi_squared_concentration}, to obtain
  \begin{equation}
    \label{eq:large_error_estimation}
    \P[\| \hat \theta - \theta \|_A \geq \|\theta\|_A]
    \leq
    e^{2d / 3 - N_e \|\theta\|^2 / (3 \sigma^2 d)}  \,.
  \end{equation}
  To see this, note that for any $x \geq 0$, applying the inequality $2\sqrt{xd}
  \leq  x + d$, we have
  \[
    \frac{\sigma^2 d}{N_e} \big( d + 2 \sqrt{xd} +  2x \big)
      \leq \frac{\sigma^2 d}{N_e} \big( 2d + 3x \big)
      \,.
  \]
  Then consider $x = N_e \|\theta\|_A^2 / (3 d \sigma^2) - 2d / 3$, so that 
  $\frac{\sigma^2 d}{N_e} \big( 2d + 3x \big)  = \|\theta \|_A^2$. If $x \leq 0$, then the
  claim~\eqref{eq:large_error_estimation} holds trivially as a probability is less than
  $1$. Otherwise, we apply subgaussian concentration
  (Lemma~\ref{lem:chi_squared_concentration}) to get
  \begin{align*}
    \P\big[\|\hat \theta - \theta\|_{A}^2 &\geq 
    \|\theta\|_A^2
    \big] 
    = 
    \P\big[\|\hat \theta - \theta\|_{X^\top X}^2 
      \geq (\sigma^2 d  / N_e) \big( 2d + 3x \big)
    \big] \\
    & \leq 
    \P\big[\|\hat \theta - \theta\|_{X^\top X}^2 \geq 
    (\sigma^2 d^2  / N_e ) \big( d + 2 \sqrt{dx} +  2 x \big)
    \big] \\
    & \leq e^{-x} 
    = e^{2d / 3 - N_e \|\theta\|^2 / (3 \sigma^2 d)}  
    \,. \qedhere
  \end{align*}
\end{proof}

\begin{lemma*}\!\emph{\ref{lem:warm-up}} \;
Let $\hat B$ be the output of the warm-up, then,
    \[
        \P\Big[
            \hat B  \notin \big[  (1 / 2)\|\theta\|_A, (3 / 2)\|\theta\|_A\big]
            \Big]
        \leq \frac{164 \, \sigma^2 d^2}{T \|\theta\|_A^2}  
        \overline \log \bigg(\frac{T\|\theta\|_A^2}{d^2 \sigma^2}\bigg)  
        + 48 \frac{d}{T} \,, 
    \]
    and the expected number of time-steps until the end the warm-up is 
    \[
    \E[N_{\warmup}]
        \leq \frac{164 \, \sigma^2 d^2}{\|\theta\|_A^2}  
        \overline \log \bigg(\frac{T\|\theta\|_A^2}{\sigma^2 d^2}\bigg) + 48d
        \,.
    \]
\end{lemma*}

\begin{proof}
Define the threshold index
\begin{equation*} \label{eq:istardef}
i^\star = \min \big\{
    i : \; U^2(\delta_i, n_i) \leq \|\theta\|^2_A
\big\} \, , 
\end{equation*}
around which the warm-up should terminate. We first prove an upper bound on the length
of the phase at that threshold index, 
  \begin{equation}
    \label{eq:bound_i_star}
    2^{i^\star} \leq  
    \frac{(2 + \sqrt 2)\sigma^2 d}{\|\theta\|_A^2} 
    \, 
    \overline \log \bigg(\frac{T\|\theta\|_A^2}{\sigma^2d^2}\bigg)
    + 1 \,.
\end{equation}

\paragraph{Proof of \eqref{eq:bound_i_star}}
    This essentially consists in inverting the formula of the definition of $U$. To this
    end define,
    \[
        \bar n = 
        \frac{(1 + \sqrt 2 / 2)\sigma^2d^2}{\|\theta\|^2} 
        \overline \log \bigg(\frac{T\|\theta\|^2}{\sigma^2d^2}\bigg)
        \,.
    \]
  Define $\log_+ (x) = \log(\max (x, 1))$. If $n_i \geq \bar n$, then, using the inequality
  $2\sqrt u \leq \sqrt 2 / 2 + (\sqrt 2 )u$ for $u \geq 0$, we get
  \begin{align*}
      U(\delta_i, n_i) 
      &=
      \frac{\sigma^2d^2}{n_i} \bigg( 1  +
      2\sqrt{\frac{1}{d} \log_+ \Big(\frac{d}{T n_i}} \Big)
      + \frac{2}{d} \log_+ \Big(\frac{d}{T n_i}\Big)\bigg) \\
      &\leq
      \frac{\sigma^2d^2}{\bar n} \bigg( 1  +
      2\sqrt{\frac{1}{d} \log_+ \! \Big(\frac{d}{T \bar n}}\Big)
      + \frac{2}{d} \log_+ \! \Big(\frac{d}{T \bar n}\Big)\bigg) \\
      &\leq
      \frac{\sigma^2d^2}{\bar n} \bigg( 2  
      + \frac{3}{d} \log_+ \! \Big(\frac{d}{T \bar n}\Big)\bigg)
      = \|\theta\|^2
      \frac{1 + \sqrt 2 / 2  
      + \frac{2 + \sqrt 2}{d} \log_+ \! \Big(\frac{d}{T \bar n}\Big)}
      {(1 + \sqrt 2 / 2)\overline \log \Big(\frac{T\|\theta\|^2}{\sigma^2d^2}\Big) } \\
      &\leq 
      \|\theta\|^2
      \frac{1 + \sqrt 2 / 2  
      + \frac{2 + \sqrt 2}{d} \log_+\!\!\Big(\frac{T\|\theta\|^2}{2\sigma^2d}\Big)}
      {(1 + \sqrt 2 / 2)\overline \log \Big(\frac{T\|\theta\|^2}{\sigma^2d^2}\Big) } \\
      & \leq 
      \|\theta\|^2
      \frac{ 1 + \sqrt 2 / 2
      + (1 + \sqrt 2 / 2)\log_+\!\!\Big(\frac{T\|\theta\|^2}{ \sigma^2 d^2}\Big)}
      {(1 + \sqrt 2 / 2)\overline \log \Big(\frac{T\|\theta\|^2}{\sigma^2 d^2}\Big) }
      = \|\theta\|^2 \, ,
  \end{align*}
  where we also used the facts that $d \geq 2$ and
  $
      1 / \bar n \leq \|\theta\|^2 / (\sigma^2 d^2) \,.
  $
  Therefore, we have shown that if $n_i \geq \bar n$, then $i \geq i^\star$, hence
  \[
  2^{i^\star}
  \leq \min \{ 2^i \; | \; n_i \geq \bar n \}
  \leq \max \Big(\frac{2 \bar n }{d}, 1\Big) 
  \leq \frac{2 \bar n }{d} + 1 \,, 
  \]
which is the claimed inequality. 

Now that \eqref{eq:bound_i_star} is proved, let us proceed with the next steps. Let
$\hati$ be the index of the warm-up round in which the algorithm terminates. Then
by design, $\|\hat \theta_{\hati}\| > 3U(\delta_{\hati}, n_{\hati})$. Therefore, by
Lemma~\ref{lem:a-bgeqa} and the triangle inequality, 
\begin{align*}
    \P\Big[ 
        \big| \|\hat \theta_{\hati }\| - \|\theta \| \big| > \frac{1}{2} \|\theta\| 
        \Big]
    &\leq \P\Big[ 
            \big| \|\hat \theta_{\hati}\| - \|\theta \| \big| 
            > \frac{1}{3} \|\hat \theta_{\hati}\|
        \Big]  \\
    &\leq  \P\Big[ | \|\hat \theta_{\hati }\| - \|\theta \| | 
        > U(\delta, n_{\hati})
            \Big]
    \leq \P\Big[  \|\hat \theta_{\hati } - \theta \| > U(\delta, n_{\hati})
            \Big] 
    \,.
\end{align*}
With a union bound on the possible values of~$\hati$, and applications of
Lemma~\ref{lem:chi_squared_concentration}, the probability above is less than
\begin{align}
     \sum_{i=1}^{i^\star + 4}  \P\Big[  \|\hat \theta_i -  \theta \| 
    > U(\delta_i, n_{i}) \Big] 
        + \P\big[\hati > i^\star + 4\big] 
    &\leq \sum_{i=1}^{i^\star + 4} \delta_i
        + \P\big[\hati > i^\star + 4\big] \nonumber \\
    &\leq \frac{d}{T} 2^{i^\star + 5} 
        + \P\big[\hati > i^\star + 4\big] 
        \,.
    \label{eq:sum_delta}
\end{align}
    The number of warm-up rounds can be written as 
    \[
      N_{\warmup} = \min \bigg( T, \sum_{i = 1}^{\hati} n_i \bigg)
    \]
    so that, doing a union bound over the value of $\hati$ as above, 
    \begin{equation} \label{eq:sum_ni}
      \E[N_{\warmup}]
        \leq \sum_{i = 1}^{i^\star + 4} n_i 
        + T \P\big[\hati > i^\star + 4\big] 
        \leq d 2^{i^\star + 5} 
        + T \P\big[\hati > i^\star + 4\big]  
        \,.
    \end{equation}
    Now observe that on the event 
    $
        \{\hati > i^\star + 4\big\}
    $, 
    we have
    \[  
        \|\hat \theta_{i^\star + 4}\| \leq 3 U(\delta_{i^\star + 4}, n_{i^\star + 4})
        \quad \text{and} \quad
        U(\delta_{i^\star + 4}, n_{i^\star + 4} )
        < \frac{1}{4} U(\delta_{i^\star}, n_{i^\star})
       \leq \frac{1}{4} \|\theta\| 
       \,.
    \]
    Therefore, by combining both inequalities, on this event, 
    \[
        \|\theta - \hat \theta_{i^\star + 4}\| 
        \geq
        \|\theta\| - \|\hat \theta_{i^\star + 4}\|
        > 4 U(\delta_{i^\star + 4}, n_{i^\star + 4} )
         - 3U(\delta_{i^\star + 4}, n_{i^\star + 4} ) 
         = U(\delta_{i^\star + 4}, n_{i^\star + 4} ) 
         \,.
    \]
    We have thus shown that 
    \[
        \P\big[\hati > i^\star + 4\big]
        \leq 
        \P\big[
        \|\theta - \hat \theta_{i^\star + 4}\| 
        > U(\delta_{i^\star + 4}, n_{i^\star + 4} ) \big]
        \leq \delta_{i^\star + 4} 
        = \frac{d}{T} 2^{i^\star + 4}
        \,.
    \]
    Plug this in \eqref{eq:sum_delta} and \eqref{eq:sum_ni} to conclude.
    \[
        \P\Big[
            \hat B  \notin \big[  (1 / 2)\|\theta\|_A, (3 / 2)\|\theta\|_A\big]
            \Big]
        \leq 
        \frac{48 d}{T}  2^{i^\star}
        \quad \text{and} \quad
        \E[N_\warmup] \leq  48 d 2^{i^\star} 
        \,.\qedhere
    \]
  \end{proof}

\subsection{Technical lemmas}

\begin{lemma}
  \label{lem:ub_inst_regret_square}
  For any $\theta, u \in \R^d$, if  $\theta^\top A u \geq 0$ then, 
  \[
    \theta^\top (x^\star(\theta) - x^\star(u))
    \leq \frac{1}{\|\theta\|_A} \|\theta - u \|_A^2
    \,.
  \]
\end{lemma}

\begin{proof}%
  Without loss of generality $A = I_d$, since $\|A^{1/2}\theta\| = \|\theta\|_A$ and 
  \[
    \theta^\top \big(x^\star_A(\theta) - x^\star_A(u)\big)
    = A^{1 /2}\theta^\top \big(x^\star_I(A^{1 /2}\theta) - x^\star_I(A^{1 /2}u)\big)
    \,.
  \]
  Then $x^\star(\theta) = (1 / \|\theta\|)\theta$ and  $x^\star(u) = (1 / \|u\|)u$, so
   \[
    \theta^\top (x^\star(\theta) - x^\star(u))
    = \frac{\|\theta\|}{2} \Big\|\frac{1}{\|\theta\|} \theta - \frac{1}{\|u\|} u \Big\|^2
    = \frac{1}{2 \|\theta\| } 
    \Big\| \theta - \frac{\|\theta\|}{\|u\|} u \Big\|^2
    \,. 
  \]
  Furthermore, using the fact that $ 0 \leq \theta^\top u \leq \|\theta\| \|u\|$, 
  \[
    \Big\| \theta - \frac{\|\theta\|}{\|u\|} u \Big\|^2
    = 2\| \theta \|^2  - 2 \frac{\|\theta\|}{\|u\|} \theta^\top u
    \leq 
    2\| \theta \|^2  - 2 \frac{1}{\|u\|^2} (\theta^\top u)^2
    \,.
  \]
  Recognizing that the right-most term is twice the distance from $\theta$ to its 
  projection on the line with direction $u$,
  \[
    \Big\| \theta - \frac{\|\theta\|}{\|u\|} u \Big\|^2
    \leq
    2 \Big( \|\theta\|^2 - \frac{1}{\|u\|^2} (\theta^\top u)^2 \Big)
    = 2 \inf_{\lambda \in \R} \| \theta  - \lambda u \|^2
    \leq 2 \| \theta - u \|^2
    \, .
    \qedhere
  \]
\end{proof}

\begin{lemma}\label{lem:a-bgeqa}
    For any $a, b, \lambda \geq 0$ such that $ |a - b| \geq \lambda a$, 
    we have $|a - b| \geq \lambda / (1 + \lambda) b$.
\end{lemma}
\begin{proof}
    We consider two cases.
    If $a \geq b$, then $|a-b| \geq \lambda a \geq \lambda b \geq \lambda / ( 1 +
    \lambda) b$ proving the claim. Otherwise, if $a < b$, then
    $
        (1 + \lambda) |a-b| 
        = |a-b| + \lambda(b-a)
        \geq \lambda a +  \lambda (b - a) = \lambda b \,.
    $
\end{proof}
\newpage

\section{Action Sets That Do Not Contain Zero}
\label{app:reduction}

\subsection{Non-centered Action Sets}
\label{sec:non-centered-ellipoids}
We now describe how to modify our algorithm to handle non-centered ellipsoids, i.e., 
sets of the form 
\[
	\mathcal{X}= \{ x \in \mathbb{R}^d: \|x-c\|_{A^{-1}} \le 1 \}
\]
for $c \neq 0$. 
\paragraph{ETC-type Algorithms}
The modification applies more generally to any ETC-type algorithm, going from action
sets $\cX$ that contain $0$ to the translated $c + \cX$ for any $c \in \R^d$.
Let us give some definitions.
\begin{definition}
 An algorithm is an ETC-type algorithm if there exists a stopping time $T_{\commit}$
 such that for any $t > T_{\commit}$, the algorithm commits to the action $x_t =
 x_{T_{\commit} + 1}$.
\end{definition}
In the case of Algorithm~\ref{alg:main-alg}, the warm-up rounds should be counted as
exploration rounds.
We denote the commit action of an ETC-type algorithm by $x_{\commit} := x_{T_{\commit}
+ 1}$. An ETC-type algorithm comes with a natural upper bound on its regret, as
\begin{equation}\label{eq:etc_reg_decomp}
  R_T(\cA, \cX, \theta)
  \leq \Big(\max_{x, y \in \cX} \|x - y\|\Big) \|\theta\|_\star \,  \E[T_{\commit}]
    + T\,  \E [ \theta^\top (x^\star(\theta) - x_{\commit}) ]
  \,, 
\end{equation}
where $\|\cdot\|$ denotes any norm on $\R^d$ and $\| \cdot\|_\star$ denotes its dual norm.

Consider a parametric ETC-type algorithm $\alg$, that takes a variance proxy and
time-horizon as inputs. Assume that for any linear bandit problem, the algorithm $\cA =
\alg(\sigma^2, T)$ tuned with $\sigma^2$ and $T$ satisfies
\begin{equation}
  \label{eq:BandC}
  \E  \big[ T_{\commit}(\cA) \big] 
  \leq B(\theta, \sigma^2, T)
  \quad \text{ and } \quad
  \E
  \big[ \theta^\top \big(
        x_{\cX}^\star(\theta) - x_{\commit}(\cA)
        \big)
        \big] 
  \leq C(\theta, \sigma^2, T)
  \,, 
\end{equation}
for some $B$ and $C$, called the expected exploration time and expected commit
error, respectively.

\paragraph{A Reduction}
We transform any ETC-type algorithm for action sets $\cX \ni 0$ into an algorithm for
the translated set $c + \cX$ which retains essentially the same bound as above up to
constant factors.

Consider a parametric algorithm $\alg$ for $\cX$ that takes the variance proxy and
horizon as inputs, and let $\cA = \alg(2 \sigma^2, T / 2)$.
The recipe consists in replacing the plays $\tilde x_t \in \cX$ in each round of
exploration made by $\mathcal A$ by two rounds of exploration, playing first the reference point $c$ and then $c +\tilde x_t$, at times $2t- 1$ and $2t$.
Now observe that 
\[
  y_{2t-1} - y_{2t}   = \theta^\top \tilde x_t + (z_{2t-1} - z_{2t})  \,.
\]
Therefore, by providing $y_{2t-1} - y_{2t}$ as inputs to $\cA$, the observations
received by $\cA$ are consistent with the linear bandit model with action set $ \cX$,
except that the sub-gaussian variance proxy is doubled.
Finally, since the optimal action is equivariant by translation of the action set,
the commit error is the same in both cases. 

\begin{algorithm} \caption{Reduction for non-centered action sets for ETC algorithms}\label{alg:non-centered} 
  \SetKwInOut{Input}{Input}
  \Input{
  ETC-style algorithm $\cA$ for action set $\cX$ containing $0$. 
  Translation vector $c \in \R^d$.}
  \KwInit{$t = 1$}
  \While{$t \leq T_\commit(\cA)$}{
    Receive $\tilde x_t \in \cX$ from $\cA$\\
    Play $x_{2t -1} = c$, receive $y_{2t - 1}$. 
    Play $x_{2t} = \tilde x_t$, receive $y_{2t}$.\\
    Give $y_{2t-1} - y_{2t}$ as input to $\cA$ \\
  }
      Get $x_{\commit}$ from $\cA$ \\
  \While{$t \leq T$}{
    Play $c + x_{\commit}$ \\
  }
\end{algorithm}

\subsection{Analysis of the Reduction}
The next theorem states that the algorithm obtained by applying the reduction enjoys
essentially the same guarantees as the initial algorithm. 

\begin{theorem}
  \label{thm:non_centered_gen}
  Let $\cX$ denote an action set containing $0 \in \R^d$. Let $\alg$ be an
  ETC type algorithm with bounds $B$ and $C$ (see \eqref{eq:BandC}). 
  Denoting by $\cA_c$ the algorithm obtained by applying the reduction
  Algorithm~\ref{alg:non-centered} to $\alg(2\sigma^2, T / 2)$ for translation vector $c
  \in \R^d$, for any linear bandit problem with parameters $\theta$ and $\sigma^2$,  
  \[
    R_T(\cA_c, c + \cX, \theta)  \leq  
    2\Big(\max_{x, y \in \cX} \|x - y\|\Big) \|\theta \|_\star 
    B(\theta, 2\sigma^2, T / 2)
    +  T \, C(\theta,  2\sigma^2, T/2) \,.
  \]
\end{theorem}
This implies in particular that if there is a good ETC-type algorithm for an action set
containing $0$, then applying a translation to the action set can only make the linear
bandit problem easier. 
Note also that the translation vector $c$ does not appear in the upper bound, perhaps 
surprisingly. 
\begin{proof}
  Let $T_{\commit}(\cA)$ denote the commit time, and $x_{\commit}(\cA) \in \cX$ the
  commit action of $\cA$ on its observations.  
  The sequence of observations given to $\cA$ fit exactly a linear bandit model with
  subvariance proxy $2 \sigma^2$ and tuned for horizon $T / 2$, so the guarantees on
  $\cA$ hold:
  \[
    \E  \big[ T_{\commit}(\cA) \big] 
    \leq B(\theta, 2\sigma^2, T / 2)
    \quad \text{ and } \quad
    \E 
    \big[ \theta^\top \big(
          x_{\cX}^\star(\theta) - x_{\commit}(\cA)
          \big)
          \big] 
    \leq C(\theta, 2\sigma^2, T / 2)
    \,.
  \]
  By construction $\cA_c$ is an ETC-type algorithm with $T_{\commit}(\cA_c) = 2
  T_{\commit}(\cA)$, and the commit action is $x_{\commit}(\cA_c) = c +
  x_{\commit}(\cA)$. So the regret decomposition \eqref{eq:etc_reg_decomp} applies and
  \begin{align*}
    R_T(\theta)
     & \leq \Big(\max_{x, y \in \cX} \|x - y\|\Big) 
    \|\theta\|_\star \,  \E[T_{\commit}(\cA_c)]
      + T\,  \E [ \theta^\top (x^\star_{c + \cX}(\theta) - x_{\commit}( \cA_c)) ] \\
    & \leq  2\Big(\max_{x, y \in \cX} \|x - y\|\Big) 
    \|\theta\|_\star \,  \E[T_{\commit}(\cA)]
      + T\,  \E [ \theta^\top (x^\star_{\cX}(\theta) - x_{\commit}( \cA)) ]
      \,. \qedhere
  \end{align*}
\end{proof}

\newpage
\section{Proof details for Section~\ref{sec:Regret lower bound}}
\label{app:lower_bound}

\begin{proof}[Proof of Proposition~\ref{prop:dBlower}]
	We consider a noiseless linear bandit problem with mean parameter $\theta$ sampled
	 from a Gaussian distribution $ \cN (0 , (B^2 /d) A^{-1})$. 
	 The idea of the proof is to lower bound the expected regret of any algorithm when 
	 averaged over this distribution.

 	After $t\leq d$ rounds of playing, let $\cF_t$ denote the $\sigma$-algebra generated 
	by the observations $y_1, \dots, y_t$, and the chosen actions
	$x_1, \dots, x_{t+1}$ (we include $x_{t+1}$ to account for the possible internal 
	randomization of the algorithm). 
    At time $t + 1$, if we integrate both over the initial randomization
    on $\theta$ and on the internal randomization, by the tower rule and since $x_{t+1}$
    is $\cF_t$-measurable,
	\[
        \E[(x_{t+1} - c)^\top \theta] 
        = \E[(x_{t+1} - c)^\top  \E[\theta | \cF_t ]] 
		\leq \E\big[\| \E \big[\theta | \cF_t \big]\|_A\big]
        \, .
	\]
	By standard results on Bayesian linear regression, (see, e.g., 
	Chapter 9.2 in \citet{hoff2009first}), 
	the distribution of $A^{1/2}\theta$ conditional on $\cF_t$ is
	$
		\cN \big( P_t A^{1/2} \theta, (B^2/d)(I_d - P_t) \big)
	$
	where $P_t$ denotes the orthogonal projection on the span of $A^{-1/2} x_1, \dots, A^{-1/2} x_t$. 
	Orthonormalizing the family $(A^{-1/2} x_t)$ yields orthonormal vectors 
	$e_1, \dots, e_t$ such that $e_i$ is $\cF_{i-1}$-measurable and
	\[
		\|A^{-1/2}P_t A^{1/2} \theta \|_A^2
		= \|P_t A^{1/2} \theta \|^2
		= \sum_{i=1}^{t} ((A^{1/2} \theta)^\top e_i)^2
		\,.
	\]
	Then, since $e_i$ is $\cF_{i-1}$-measurable, the vector $(A^{1/2}\theta)^\top
	 e_i$ is normally distributed with distribution
	 $
		\cN \big( 0, (B^2/d)e_i^\top (I_d - P_{i-1})e_{i}\big)
		=
		\cN \big( 0, B^2/d \big).
	$
	So for all $i \leq d$, we have the equality
	$
	\E \big[ ((A^{1/2} \theta)^\top e_i)^2 | \cF_{i-1} \big]
	= B^2 / d \,.
	$
	Applying this repeatedly, we see that
	\[
		\E [\|A^{-1/2}P_tA^{1/2}\theta \|_A] 
		\leq \sqrt{\E \big[\|A^{-1/2} P_t A^{1/2} \theta \|_A^2\big]} 
		\leq  B\sqrt{\frac{t}{d}}
		\,.
	\]
	This implies that the averaged regret up to time $d$ is at least
	\begin{multline*}
		\E [R_d(\theta)]
		= \sum_{t=1}^{d} \E[(x^\star(\theta) - c)^\top \theta] - \E[(x_t- c)^\top \theta ]
		\\
		\geq \sum_{t=1}^{d} \E[\|\theta\|_A] - \E[\|A^{-1/2} P_{t-1} A^{1/2} \theta\|_A]
		\geq \sum_{t=1}^{d} \E[\|\theta\|_A] - B \sqrt{\frac{t-1}{d} }
		\,.
	\end{multline*}
	Note also that since $(\sqrt{d} / B) A^{1/2} \theta$ is a standard gaussian vector, 
	\begin{align*}
		\mathbb{E}_{\theta}\big[ \|A^{1/2}\theta\| \big]
		&= \frac{B}{\sqrt d} \frac{\sqrt{2} \Gamma((d+1)/2)}{\Gamma(d/2)}  
		  \geq B \sqrt{(d-1) / d} \,, 
	\end{align*}
	where we used Gautschi's inequality, 
	\begin{align*}
		x^{1-s}	\le \frac{\Gamma(x+1)}{\Gamma(x+s)}, 
	\end{align*}
	with $x = (d-1) /2 $ and $s = 1/2$. 
	All in all, this implies that
	\[
		\E [R_T(\theta)	]
		\geq  B \sum_{t=1}^{d} \sqrt{\frac{d-1}{d}} - \sqrt{\frac{t-1}{d} }
		\geq \frac{1}{3} d B
		\,.
	\]
	To complete the averaging argument, we also need to ensure that the norm 
    of $\theta$ is large enough. Let $c$ be some numerical parameter, of which we will set
    the value later on. By a case disjunction, 
	\[
		\E \big[R_d(\theta) \mathds{1}\{\|\theta\|_A \leq c B \}\big]
		= 
		\E [R_d(\theta)] 
		-
		\E [R_d(\theta) \mathds{1}\{\|\theta\|_A > c B \}]
		\,.
	\]
	We apply a version of the peeling trick decomposing 
	over the values of $\|\theta\|_A$ in a grid to bound
	\begin{multline*}
		\E [R_d(\theta) \mathds{1}\{\|\theta\|_A > c B \}]
		 \leq \sum_{i = 1}^{+\infty} 2d \sqrt{i+1} cB 
			\P[\|\theta\|_A > \sqrt i c B ] \\
		\leq 2\sqrt 2 dB \sum_{i = 1}^{+\infty}  
		\sqrt{c^2i} \P[\|\theta\|_A > \sqrt i c B ] 
		= \frac{2\sqrt 2 dB}{\sqrt d} \sum_{i = 1}^{+\infty}  
		\sqrt{dic^2} \P[\|\theta\|_A > \sqrt i c B ]
		\,.
	\end{multline*}
	Use the concentration inequality Lemma~\ref{lem:chi_squared_concentration}, which 
	applies in particular to the Gaussian vector $A^{1/2}\theta$, 
	\[
		\P\Big[
			\|A^{1/2}\theta\|^2 \geq B^2 \Big(2 + \frac{3x}{d} \Big)
		\Big]
		\leq e^{-x}
		\quad \text{with} \quad
		x = \frac{d}{3} (c^2 \,i  - 2)
		\,.
	\]
	If $c^2 \geq 8$, then $c^2 i - 2 \geq 3 c^2i / 4$ as $c^2 i / 4 \geq 2$
	(remember $i \geq 1$) so
	\[
		\P\big[
			\|A^{1/2}\theta\|^2 \geq ic^2B^2  
		\big]
		\leq \exp \Big( - \frac{idc^2}{4}\Big)
	\]
	Note furthermore that since the function $xe^{-x^2/8}$ is decreasing on 
	$[4, +\infty)$, so for any $x \geq 4$, 
	$
		x e^{- x^2 / 4} \leq 4e^{-2} e^{- x^2 / 8} \,.
	$
	Consequently, if $c \geq 2 \sqrt 2$, then $\sqrt {idc^2} \geq 4$ since
	 $d \geq 2$, and  
	\[
		\sum_{i = 1}^{+\infty}  
		\sqrt{id  c^2} 
		e^{ - idc^2 / 4}
		\leq 
		\sum_{i = 1}^{+\infty}  
		e^{ - idc^2 / 8}
		= \frac{e^{-dc^2 / 8}}{ 1- e^{-dc^2/ 8}}
		\,.
	\]
	We have thus lower bounded the regret on the event that $\|\theta\|_A \leq cB$ as
	\begin{multline*}
		\E [R_d(\theta) \mathds{1}\{\|\theta\|_A \leq c B \}]
		\geq \frac{dB}{3} -
		\frac{2\sqrt 2 dB}{\sqrt d} 
		 \frac{e^{-dc^2 / 8}}{ 1- e^{-dc^2/ 8}}
		=  \bigg(\frac{1}{3} - 
		\frac{2\sqrt 2}{\sqrt d} 
		 \frac{e^{-dc^2 / 8}}{ 1- e^{-dc^2/ 8}}
		\bigg) dB \\
		\geq  \bigg(\frac{1}{3} - 
		\frac{2 e^{-c^2 / 4}}{ 1- e^{-c^2/ 4}}
		\bigg) dB \,, 
	\end{multline*}
	we also used the facts that $d \geq 2$ and $c \geq 2 \sqrt 2$. Applying our result
	exactly with $ c = 4$, for which the last constant is larger than $0.07$, we deduce the existence of some $\theta$ such that
	 $\|\theta\|_A \leq 4 B$ and
	\[
		R_T(\theta) 
		\geq R_d(\theta)
		\geq 0.07 dB
		\geq \frac{0.07}{4}  d \|\theta\|_A 
		\geq 0.017 d\|\theta\|_A 
		\,.
		\qedhere
	\]
\end{proof}

\clearpage
\section{Glossary}
\printnoidxglossaries

\end{document}